\documentclass{article}

\usepackage[preprint]{neurips_2026}

\usepackage{xcolor}
\usepackage[utf8]{inputenc}
\usepackage[T1]{fontenc}
\usepackage{hyperref}
\usepackage{url}
\usepackage{microtype}
\usepackage{graphicx}
\usepackage{subcaption}
\usepackage{booktabs}
\usepackage{multirow}
\usepackage{amsmath}
\usepackage{amssymb}
\usepackage{mathtools}
\usepackage{amsthm}
\usepackage[capitalize,noabbrev]{cleveref}

\theoremstyle{plain}
\newtheorem{theorem}{Theorem}[section]
\newtheorem{proposition}[theorem]{Proposition}
\newtheorem{lemma}[theorem]{Lemma}
\newtheorem{corollary}[theorem]{Corollary}
\theoremstyle{definition}
\newtheorem{definition}[theorem]{Definition}
\newtheorem{assumption}[theorem]{Assumption}
\theoremstyle{remark}
\newtheorem{remark}[theorem]{Remark}

\newcommand{\R}{\mathbb{R}}
\newcommand{\E}{\mathbb{E}}

\title{Linearized Attention Cannot Enter the Kernel Regime at Any Practical Width}

\author{%
  Jose Marie Antonio Minoza \\
  Center for AI Research PH \\
  \And
  Paulo Mario P. Medina \\
  Center for AI Research PH \\
  \And
  Sebastian C. Iba\~{n}ez \\
  Center for AI Research PH \\
}

\begin{document}

\maketitle

% ==============================================================================
% ABSTRACT
% ==============================================================================
\begin{abstract}
Understanding whether attention mechanisms converge to the kernel regime is foundational to the validity of influence functions for transformer accountability. Exact NTK characterization of softmax attention is precluded by its exponential nonlinearity; linearized attention is the canonical tractable proxy and the object of study here. This paper establishes that even this proxy does not converge to its NTK limit at any practical width, revealing a fundamental trade-off in the learning dynamics of attention. An exact correspondence is established between parameter-free linearized attention and a data-dependent Gram-induced kernel; spectral amplification analysis shows that the attention transformation cubes the Gram matrix's condition number, requiring width $m = \Omega(\kappa_d(\mathbf{G})^6 n\log n)$ for NTK convergence, where $\kappa_d(\mathbf{G})$ is the effective condition number of the rank-$\min(n,d)$ truncation of the input Gram matrix; for natural image datasets this threshold is physically infeasible ($m \gg 10^{24}$ for MNIST and $m \gg 10^{29}$ for CIFAR-10, 12--17 orders of magnitude beyond the largest known architectures). \emph{Influence malleability} is introduced to characterize this non-convergence: linearized attention exhibits 2--9$\times$ higher malleability than ReLU networks under adversarial data perturbation, with the gap depending on dataset condition number and task setting. A dual implication is established: the same data-dependent kernel is shown theoretically to reduce approximation error when targets align with the data geometry, while, empirically, creating vulnerability to adversarial manipulation of the training data. The structural argument extends to trainable QKV attention under standard initialization, with direct consequences for influence methods applied to deployed transformer architectures.
\end{abstract}

% ==============================================================================
% INTRODUCTION
% ==============================================================================
\section{Introduction}
\label{sec:intro}

Attention mechanisms have revolutionized deep learning across domains, yet the learning processes that give attention its characteristic flexibility lack rigorous theoretical characterization. Conventional approaches focus on architectural properties at initialization or final performance, missing the crucial dynamics of how attention learns.

Recent advances in Neural Tangent Kernel (NTK) theory have enabled precise analysis of neural network learning dynamics through kernel methods \citep{jacot2018neural, lee2019wide}, yet attention mechanisms have remained largely outside this theoretical framework. The NTK framework predicts that sufficiently wide networks operate in a \emph{lazy training} regime where the kernel remains approximately constant during training, enabling kernel-based analysis.

This connection carries direct practical consequences. Modern influence methods (TRAK \citep{park2023trak}, DataInf \citep{kwon2023datainf}, EK-FAC \citep{grosse2023ekfac}) share a foundational assumption: the model operates near the kernel regime, so the first-order NTK approximation holds throughout training. When applied to transformers, this assumption is taken for granted; practitioners use influence scores to audit training data, detect poisoning, and select representative examples \citep{koh2017understanding, zhang2022rethinking}. \textbf{Central question.} \emph{When, if ever, does linearized attention enter the kernel regime, and what follows for model accountability if it does, or if it structurally cannot?} If convergence holds, influence methods applied to attention inherit their theoretical guarantees; if convergence fails at every practical width, those scores are ungrounded and exploitable. This paper provides the first structural answer.

Figure~\ref{fig:ntk_convergence} confirms the prediction of Theorem~\ref{thm:cubic_conditioning}: standard two-layer ReLU networks exhibit the expected monotonic decrease in NTK distance $\|f_m - f_{\text{NTK}}\|$ as width increases, while attention-enhanced architectures show fundamentally different behavior. NTK distance fails to decrease and instead remains high or increases, with trajectory varying by dataset (non-monotonic on MNIST, monotonically increasing on CIFAR-10). Both exhibit the key signature of non-convergence. This reveals that attention operates in the \emph{feature learning} regime \citep{chizat2019lazy}, where representations evolve substantially during training.

\begin{figure*}[t]
\centering
\includegraphics[width=\textwidth]{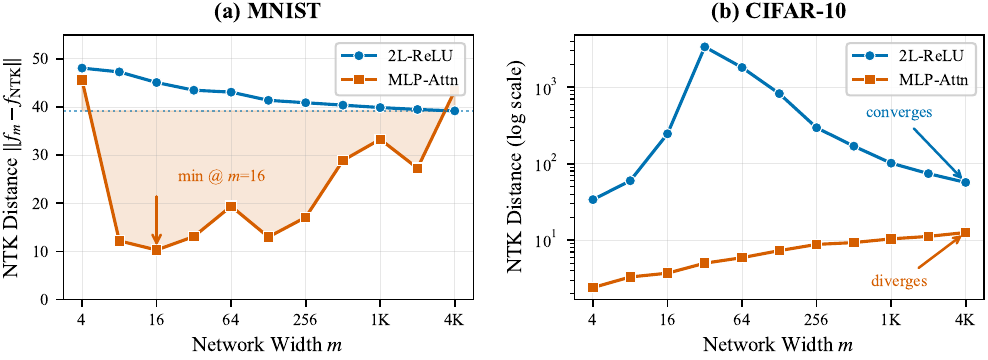}
\caption{NTK distance $\|f_m - f_{\text{NTK}}\|$ across network widths, where $f_m$ is the finite-width trained model and $f_{\text{NTK}}$ is the infinite-width NTK predictor. \textbf{2L-ReLU} (blue) shows expected convergence: distance decreases as $m \to \infty$. \textbf{MLP-Attn} (orange), using \emph{linearized} attention ($f^{\text{att}}(\mathbf{X}) = \mathbf{X}\mathbf{X}^T\mathbf{X}$) not full softmax, shows fundamentally different behavior: distance fails to decrease monotonically on either dataset (non-monotonic on MNIST, increasing on CIFAR-10), consistent with the non-convergence established in Theorem~\ref{thm:cubic_conditioning} (see Remark~\ref{rmk:convergence_cases} for the convergence/non-convergence/divergence taxonomy). This figure is a diagnostic illustration; Theorem~\ref{thm:cubic_conditioning} is a mathematical result that does not require empirical validation at these scales.}
\label{fig:ntk_convergence}
\end{figure*}

To quantify the implications of this non-convergence, \emph{influence malleability} is introduced, defined as the capacity to dynamically alter reliance on training examples as their quality changes. The analysis reveals that attention exhibits 2--9$\times$ higher influence malleability compared to standard ReLU networks, providing a measurable signature of attention's sensitivity to training data.

\subsection{Contributions}

\begin{itemize}
\item \emph{NTK non-convergence and spectral amplification}: Linearized attention is shown empirically and theoretically to not converge to its infinite-width NTK limit (Figure~\ref{fig:ntk_convergence}). Theorem~\ref{thm:cubic_conditioning} establishes that attention cubes the condition number, requiring $m = \Omega(\kappa_d(\mathbf{G})^6 n\log n)$, exceeding $10^{24}$ for MNIST and $10^{29}$ for CIFAR-10. Theorem~\ref{thm:polynomial_correspondence} provides the exact Gram-induced kernel enabling the analysis, resolving a convergence question explicitly deferred by Tensor Programs \citep{yang2020tensor} (Remark~7.4) for non-standard (non-backpropagation-like) architectures.

\item \emph{Influence malleability}: Influence malleability is introduced as a quantifiable signature of attention's sensitivity to training data: linearized attention exhibits 2--9$\times$ higher flip rates than ReLU networks (up to 6--9$\times$ in multiclass; attenuated for low-$\kappa_d(\mathbf{G})$ configurations).

\item \emph{Dual implications}: The data-dependent kernel reduces approximation error when targets align with data geometry (Proposition~\ref{prop:bias_reduction}), but the same sensitivity creates vulnerability to adversarial data manipulation, a tradeoff with a common origin in NTK non-convergence. Proposition~\ref{prop:qkv_generalization} extends the result to trainable QKV under standard initialization.
\end{itemize}

% ==============================================================================
% RELATED WORK
% ==============================================================================
\section{Related Work}
\label{sec:related}

\textbf{Neural Tangent Kernel Theory.} The NTK framework \citep{jacot2018neural} establishes that infinitely wide networks trained with gradient descent become equivalent to kernel methods. Extensions address generalization \citep{arora2019exact}, depth scaling \citep{du2019gradient}, and arbitrary architectures \citep{yang2020tensor}. However, \citet{wenger2023disconnect} demonstrate that NTK theory only applies to networks \emph{orders of magnitude wider than deep,} a condition practical architectures violate.

\textbf{Lazy vs.\ Feature Learning.} \citet{chizat2019lazy} distinguish \emph{lazy training} (kernel regime, frozen features) from \emph{feature learning} (rich regime, evolving representations). The empirical finding that attention's NTK distance \emph{increases} with width suggests it operates firmly in the feature learning regime.

\textbf{NTK-Based Influence Functions.} \citet{zhang2022rethinking} establish that NTK-based influence computation provides superior stability compared to classical inverse-Hessian-vector-product (IHVP) methods, enabling reliable measurement of training data dependencies. \citet{loo2022evolution} study how influence functions evolve during training, showing that data importance shifts as networks learn, complementing the present focus on \emph{architectural} differences in influence dynamics.

\textbf{Feature Learning Theory.} \citet{nichani2025provable} provide provable guarantees for neural network feature learning, establishing conditions under which networks move beyond the kernel regime. The present work provides empirical evidence that attention mechanisms naturally satisfy these conditions, offering a concrete architectural mechanism for the abstract feature learning phenomenon.

\textbf{Linearized Attention.} \citet{ahn2024linear} show that linearized attention reproduces Transformer optimization dynamics. \citet{choromanski2020rethinking} provide efficient implementations via Performers. \citet{hron2020infinite} extend NTK theory to attention, showing multi-head attention behaves as Gaussian processes as heads $\to \infty$. Critically, \citet{hron2020infinite} does not identify the cubic conditioning structure $\kappa(\tilde{\mathbf{G}}) = \kappa(\mathbf{G})^3$ (with $\tilde{\mathbf{G}} = \mathbf{G}^3$; Theorem~\ref{thm:polynomial_correspondence}) established here; \citet{yang2020tensor} (Remark 7.4) explicitly defers NTK convergence for non-BP-like architectures, of which linearized attention is an instance. The present work resolves that deferred case. Two independent works corroborate the structural conclusion from the softmax side: \citet{sakai2025nngp} show that the infinite-width limit of a single softmax attention layer is non-Gaussian (requiring a new proof technique outside standard Tensor Programs), and \citet{bae2022ifsanswer} demonstrate empirically that influence estimates are unstable in real softmax transformers, consistent with Proposition~\ref{prop:softmax_amplification}.

% ==============================================================================
% METHODOLOGY
% ==============================================================================
\section{Methodology}
\label{sec:method}

\textbf{Motivation.} Modern influence methods (TRAK \citep{park2023trak}, DataInf \citep{kwon2023datainf}, EK-FAC \citep{grosse2023ekfac}) share a foundational assumption: the model operates near the kernel regime so the first-order NTK approximation holds. Whether this assumption is achievable for attention is a structural audit question. Theorem~\ref{thm:cubic_conditioning} answers it: not at any practical width for natural image datasets.

\subsection{Linearized Attention Architecture}

A parameter-free attention mechanism is designed that admits exact kernel characterization.

\begin{definition}[Linearized Attention]
For input matrix $\mathbf{X} \in \R^{n \times d}$, the linearized attention mechanism is:
\begin{equation}
f^{\text{att}}(\mathbf{X}) = \mathbf{X}\mathbf{X}^T\mathbf{X}
\end{equation}
\end{definition}

\begin{remark}[Correspondence to Standard Attention]
\label{rmk:qkv_correspondence}
This corresponds to scaled dot-product attention with identity projections ($\mathbf{W}_Q = \mathbf{W}_K = \mathbf{W}_V = \mathbf{I}$) and linearized softmax ($\exp(A_{ij}) \approx 1 + A_{ij}$). Proposition~\ref{prop:qkv_generalization} shows that full-rank projections preserve all structural properties, so identity projections are a canonical representative without loss of generality.
\end{remark}

This operation implements kernel-weighted feature aggregation, where $\mathbf{X} \in \R^{n \times d}$ denotes the \emph{entire} training set; the transformation is computed once over all $n$ training examples, making the architecture transductive. In practice, the rows of $f^{\text{att}}(\mathbf{X})$ are $\ell_2$-normalized before being passed to the MLP to ensure comparable input magnitudes across architectures. The complete architecture combines this with a two-layer MLP:
\begin{equation}
f_{\text{MLP-Attn}}(\mathbf{X}) = \frac{1}{\sqrt{m}} \sum_{r=1}^{m} a_r \sigma(f^{\text{att}}(\mathbf{X})\,\mathbf{w}_r)
\end{equation}
where $\sigma$ is ReLU, $\mathbf{w}_r \sim \mathcal{N}(0, \kappa^2 \mathbf{I})$ with $\kappa > 0$ controlling initialization scale, and $a_r \in \{-1, +1\}$ are fixed.

\subsection{Influence Malleability}

The data-dependent kernel structure of linearized attention (Section~\ref{sec:method}) suggests that attention should exhibit heightened \emph{sensitivity} to training data. To quantify this, \emph{influence malleability} is introduced: it captures how readily a model's reliance on specific training examples changes under perturbation. High malleability indicates high sensitivity; low malleability indicates rigid, fixed data dependencies.

\begin{definition}[Influence Function]
\label{def:influence}
Following \citet{koh2017understanding} and the NTK-based formulation of \citet{zhang2022rethinking}, the influence of removing training example $(\mathbf{x}_i, y_i)$ on a test prediction at $\mathbf{x}_{\text{test}}$ with true label $y_{\text{test}}$ is:
\begin{equation}
I_i \coloneq I(\mathbf{x}_i, \mathbf{x}_{\text{test}}) = \frac{1}{2}\left[(f^{\backslash i}(\mathbf{x}_{\text{test}}) - y_{\text{test}})^2 - (f(\mathbf{x}_{\text{test}}) - y_{\text{test}})^2\right]
\end{equation}
where $f^{\backslash i}$ denotes the model trained without point $i$; write $I_i^{\text{orig}}$ for $I_i$ under the original training set and $I_i^{\text{adv}}$ after perturbation. Positive influence indicates removing the point increases loss (helpful example); negative indicates removing it decreases loss (harmful example). This is computed efficiently via leave-one-out formulas using the \emph{empirical} finite-width kernel matrix $(\mathbf{K}_m + \lambda\mathbf{I})^{-1}$ without retraining. Importantly, the empirical NTK $\mathbf{K}_m$ is well-defined regardless of whether the architecture converges to its infinite-width limit; it captures how each model actually represents training data relationships at practical widths.
\end{definition}

\begin{definition}[Influence Malleability]
\label{def:influence_malleability}
Given training set $\mathcal{S}$, perturbation budget $\varepsilon > 0$, and top-$\tau$ high-influence set $\mathcal{H}_\tau = \{i : I_i^{\text{orig}} > \text{quantile}_{1-\tau}\}$ (with $\tau = 0.1$, i.e., the top 10\% by influence score), generate adversarial examples $\mathbf{x}_i^{\text{adv}}$ via Projected Gradient Descent (PGD, $\|\mathbf{x}_i^{\text{adv}} - \mathbf{x}_i\|_\infty \leq \varepsilon$) and compute perturbed scores $I_i^{\text{adv}}$. The \textbf{Influence Flip Rate} is:
\begin{equation}
\mathrm{FlipRate} = \frac{|\{i \in \mathcal{H}_\tau : \mathrm{sign}(I_i^{\mathrm{orig}}) \neq \mathrm{sign}(I_i^{\mathrm{adv}})\}|}{|\mathcal{H}_\tau|}
\end{equation}
\end{definition}

The flip rate captures the capacity to \emph{re-evaluate} which examples help versus hurt prediction, which is the signature of sensitivity to training data. PGD with $\varepsilon = 0.03$ ($L_\infty$) is used throughout. As a complementary measure, \textbf{influence stability} is assessed via Spearman's rank correlation $\rho$ between original and perturbed influence rankings; lower $\rho$ indicates greater malleability. Three data intervention strategies are evaluated: \emph{Curated} (removing the top-$\tau$ influential examples from training), \emph{Transformed} (replacing them with adversarial versions), and \emph{Adversarial} (applying PGD perturbations to all training data).

\textbf{Practical Implications for Influence Methods.}
The non-convergence result carries direct practical consequences for deploying TRAK \citep{park2023trak}, DataInf \citep{kwon2023datainf}, and EK-FAC \citep{grosse2023ekfac} on attention architectures. First, estimate $\kappa_d(\mathbf{G})$ via randomized SVD \emph{on the specific class subset} \citep{halko2011finding}: values exceeding $10^2$ place the dataset in the non-convergence regime (MNIST: $\kappa_d \approx 10^3$; CIFAR-10 multiclass: $\kappa_d \approx 8.7\times 10^3$); binary homogeneous subsets may show much lower $\kappa_d$ and ${\approx}\,1\times$ gap (Table~\ref{tab:malleability}, Binary row). Second, the 6--9$\times$ higher flip rates in multiclass settings (Table~\ref{tab:malleability}) indicate structurally unstable rankings: use lower thresholds, ensemble across checkpoints, and re-rank periodically. Adversarial training achieves comparable malleability in ReLU networks (Table~\ref{tab:adv_training}); Appendix~\ref{app:practitioner} gives step-by-step guidance.

% ==============================================================================
% THEORETICAL ANALYSIS
% ==============================================================================
\section{Theoretical Analysis}
\label{sec:theory}

The mechanistic chain connecting linearized attention to influence malleability is:
$$\text{lin.\ attn} \xrightarrow{\text{Thm}~\ref{thm:polynomial_correspondence}} \tilde{\mathbf{G}} = \mathbf{G}^3 \xrightarrow{\text{Thm}~\ref{thm:cubic_conditioning}} \kappa_d(\tilde{\mathbf{G}}) = \kappa_d(\mathbf{G})^3 \xrightarrow{\text{Thm}~\ref{thm:influence_stability}\text{ (contrapositive)}} \text{influence instability}$$
Theorem~\ref{thm:influence_stability} establishes that influence stability scales as $\epsilon/[\lambda(\lambda-\epsilon)]$. For natural image data, cubic conditioning drives $\lambda_1(\mathbf{K}) \sim \lambda_1(\mathbf{G})^3 \gg 1$, while practical regularization is set at $\lambda \sim 10^{-3} \ll \lambda_1(\mathbf{K})$; the effective condition number $\kappa(\mathbf{K}+\lambda\mathbf{I}) \approx \lambda_1(\mathbf{K})/\lambda$ is thus enormous, making the Theorem~\ref{thm:influence_stability} bound vacuous and the influence vector structurally unstable.
The $\tilde{\mathbf{G}}=\mathbf{G}^3$ identity is a property of the cube-feature-map family (Lemma~\ref{lem:spectral_transfer}); attention's NTK non-convergence follows as a corollary under the linearization, with independent QKV projections addressed in Proposition~\ref{prop:qkv_generalization}.

\begin{remark}[Convergence taxonomy]
\label{rmk:convergence_cases}
Let $d_m \coloneq \|f_m - f_{\text{NTK}}\|_2$ denote the NTK distance at width $m$. Three behaviors: (i)~\emph{Convergence} ($d_m \to 0$): 2L-ReLU at $m_0 = O(\kappa(\mathbf{G})^2 n\log n)$. (ii)~\emph{Non-convergence} ($d_m \geq \delta$ for all practical $m$, some $\delta>0$): what Theorem~\ref{thm:cubic_conditioning} establishes. (iii)~\emph{Divergence} ($d_m \to \infty$): not claimed; MLP-Attn remains structured and far from its limit, not unbounded.
\end{remark}

\begin{theorem}[Data-Dependent Gram-Induced Kernel]
\label{thm:polynomial_correspondence}
Let $\mathbf{X} \in \R^{n \times d}$ with $\|\mathbf{x}_i\|_2 = 1$ for all $i$. The kernel induced by linearized attention $f^{\text{att}}(\mathbf{X}) = \mathbf{X}\mathbf{X}^T\mathbf{X}$ is:
\begin{equation}
K_{\text{LinAttn}}(\mathbf{x}_i, \mathbf{x}_j) = \sum_{k,\ell=1}^{n} (\mathbf{x}_i^T \mathbf{x}_k)(\mathbf{x}_k^T \mathbf{x}_\ell)(\mathbf{x}_\ell^T \mathbf{x}_j)
\end{equation}
where $k$ and $\ell$ are summation indices over all $n$ training samples $\mathbf{x}_1, \ldots, \mathbf{x}_n$ in $\mathbf{X}$. This is a data-dependent kernel induced by the Gram matrix $\mathbf{X}\mathbf{X}^T$, exhibiting transitive similarity structure: influence flows from $\mathbf{x}_i$ through intermediate points $\mathbf{x}_k, \mathbf{x}_\ell$ to $\mathbf{x}_j$.
\end{theorem}

\emph{Proof.} Expanding $[f^{\text{att}}(\mathbf{X})]_i = \sum_k (\mathbf{x}_i^T\mathbf{x}_k)\mathbf{x}_k$ and taking the inner product yields $\mathbf{K}_{\text{LinAttn}} = \mathbf{G}^3$ where $\mathbf{G} = \mathbf{X}\mathbf{X}^T$. See Appendix~\ref{app:proofs}.

\begin{theorem}[NTK for Sequential Architecture]
\label{thm:ntk_correspondence}
Consider the MLP-Attn architecture $g(\mathbf{x}; \boldsymbol{\theta}) = \frac{1}{\sqrt{m}} \sum_{r=1}^{m} a_r \sigma(\mathbf{w}_r^T f^{\text{att}}(\mathbf{x}))$ with $\mathbf{w}_r \sim \mathcal{N}(0, \kappa^2 \mathbf{I})$ and fixed $a_r \in \{-1, +1\}$. In the infinite-width limit $m \to \infty$, the NTK converges to:
\begin{equation}
K_{\text{seq}}(\mathbf{x}, \mathbf{x}') = \E_{\mathbf{w}}[\sigma'(\mathbf{w}^T \tilde{\mathbf{x}})\sigma'(\mathbf{w}^T \tilde{\mathbf{x}}')] \cdot \langle \tilde{\mathbf{x}}, \tilde{\mathbf{x}}' \rangle
\end{equation}
where $\tilde{\mathbf{x}} = f^{\text{att}}(\mathbf{x})$ and $\sigma' = \mathbf{1}[\cdot > 0]$ is the ReLU derivative.
\end{theorem}

\emph{Proof.} Since $f^{\text{att}}$ is parameter-free, only $\mathbf{w}_r$ gradients contribute; the law of large numbers as $m \to \infty$ gives the stated form. See Appendix~\ref{app:proofs}.

\begin{theorem}[Influence Function Stability]
\label{thm:influence_stability}
Let $\mathbf{K} \in \R^{n \times n}$ be a positive semidefinite (PSD) NTK matrix with eigenvalues $\lambda_1 \geq \cdots \geq \lambda_n \geq 0$ and $\lambda > 0$ regularization. Then $\kappa(\mathbf{K} + \lambda\mathbf{I}) \leq \lambda_1/\lambda + 1$, and for $\|\Delta \mathbf{K}\|_2 \leq \epsilon < \lambda$:
$$\|(\mathbf{K} + \lambda\mathbf{I})^{-1} - (\mathbf{K} + \Delta\mathbf{K} + \lambda\mathbf{I})^{-1}\|_2 \leq \frac{\epsilon}{\lambda(\lambda - \epsilon)}\,.$$
\end{theorem}

\emph{Proof.} The resolvent identity gives $A^{-1} - B^{-1} = A^{-1}(B - A)B^{-1}$. Here $A = \mathbf{K}+\lambda\mathbf{I}$ satisfies $\|A^{-1}\|_2 \leq 1/\lambda$; $B = \mathbf{K}+\Delta\mathbf{K}+\lambda\mathbf{I}$ has $\lambda_{\min}(B) \geq \lambda - \varepsilon$ (since $\|\Delta\mathbf{K}\|_2 \leq \varepsilon$), so $\|B^{-1}\|_2 \leq 1/(\lambda-\varepsilon)$. Taking operator norms: $\|A^{-1}(B-A)B^{-1}\|_2 \leq \varepsilon/(\lambda(\lambda-\varepsilon))$. See Appendix~\ref{app:proofs}.

\begin{assumption}[$\mu$-Incoherence]
\label{ass:incoherence}
$\|\mathbf{x}_i\|_2 = 1$ for all $i$, and $\max_{i \neq j} |\mathbf{x}_i^T \mathbf{x}_j| \leq \mu/\sqrt{d}$ for $\mu = O(1)$. This holds with probability $\geq 1 - n^2 e^{-cd}$ for isotropic sub-Gaussian inputs with $d = \Omega(\log n)$ \citep{tropp2012user}.
\end{assumption}

\begin{proposition}[Data-Dependent Kernel Sensitivity]
\label{prop:kernel_sensitivity}
The linearized attention kernel $K_{\text{LinAttn}}$ exhibits \emph{data-dependent} sensitivity to input perturbations that depends on the Gram matrix structure. For perturbation $\delta$ with $\|\delta\|_2 \leq \epsilon$:
\begin{equation}
\left| K_{\text{LinAttn}}(\mathbf{x}_i + \delta, \mathbf{x}_j) - K_{\text{LinAttn}}(\mathbf{x}_i, \mathbf{x}_j) \right| \leq 6\lambda_1(\mathbf{G})^2\,\epsilon
\end{equation}
where $\mathbf{G} = \mathbf{X}\mathbf{X}^T$ and $\lambda_1(\mathbf{G})$ is its spectral norm. The key distinction from polynomial kernels $K_{\text{poly}}(\mathbf{x}, \mathbf{y}) = (\mathbf{x}^T\mathbf{y})^p$ (which have $O(\epsilon)$ sensitivity independent of data) is that $K_{\text{LinAttn}}$ sensitivity grows with the spectral structure of the dataset through $\lambda_1(\mathbf{G})$.
\end{proposition}

\emph{Proof.} Since $K_{\text{LinAttn}}(\mathbf{x}_i,\mathbf{x}_j) = [\mathbf{G}^3]_{ij}$, a perturbation $\|\delta\|_2 \leq \epsilon$ changes $\mathbf{G}$ by $\Delta\mathbf{G}$ with $\|\Delta\mathbf{G}\|_2 \leq 2\epsilon$ (using $\|\mathbf{x}_i\|_2=1$). By the identity $A^3-B^3 = (A-B)A^2+B(A-B)A+B^2(A-B)$ and submultiplicativity, $\|\mathbf{G}^3-(\mathbf{G}+\Delta\mathbf{G})^3\|_2 \leq 3\lambda_1(\mathbf{G})^2\|\Delta\mathbf{G}\|_2 \leq 6\lambda_1(\mathbf{G})^2\epsilon$. For polynomial kernels, $|\nabla_\mathbf{x}(\mathbf{x}^T\mathbf{y})^p|\cdot\epsilon = O(\epsilon)$, independent of $n$. See Appendix~\ref{app:proofs}.

\begin{remark}[Connection to Feature Learning]
\label{rmk:feature_learning}
Although $\mathbf{X}$ is fixed during training, $f^{\text{att}}(\mathbf{X}) = \mathbf{X}\mathbf{X}^T\mathbf{X}$ encodes global pairwise relationships. When MLP weights are updated, gradients propagate through this globally-coupled representation, causing the empirical NTK to evolve in ways that depend on the full correlation structure, the hallmark of feature learning, in contrast to frozen-kernel lazy training.
\end{remark}

The following theorem formalizes why linearized attention resists convergence to the kernel regime.

\begin{theorem}[Spectral Amplification and NTK Non-Convergence]
\label{thm:cubic_conditioning}
Let $\mathbf{X} \in \R^{n \times d}$ with $\|\mathbf{x}_i\|_2 = 1$ and $r = \operatorname{rank}(\mathbf{X}) \leq \min(n,d)$. Let $\mathbf{G} = \mathbf{X}\mathbf{X}^T$ have positive eigenvalues $\lambda_1 \geq \cdots \geq \lambda_r > 0$ and define the effective condition number $\kappa_d(\mathbf{G}) = \lambda_1/\lambda_r$. Assume the $\mu$-incoherence condition (Assumption~\ref{ass:incoherence}). Then the attention-transformed Gram matrix satisfies $\tilde{\mathbf{G}} = \mathbf{G}^3$, amplifying the effective condition number:
\begin{equation}
\kappa_d(\tilde{\mathbf{G}}) = \kappa_d(\mathbf{G})^3
\end{equation}
Consequently, for the finite-width NTK of the MLP-Attn architecture to achieve approximation error $\|K_{m,\text{seq}} - K_{\infty,\text{seq}}\|_2 \leq \varepsilon$, the width must satisfy $m = \Omega(\kappa_d(\mathbf{G})^6 n\log n / \varepsilon^2)$. In contrast, for 2L-ReLU the requirement is $m = \Omega(n\log n/\varepsilon^2)$ without spectral amplification.
\end{theorem}

\emph{Proof.} Let $\mathbf{X} = \mathbf{U}\boldsymbol{\Sigma}\mathbf{V}^T$ be the compact SVD with $r = \mathrm{rank}(\mathbf{X})$ nonzero singular values. The attention output is $\tilde{\mathbf{X}} = \mathbf{U}\boldsymbol{\Sigma}^3\mathbf{V}^T$, so $\tilde{\mathbf{G}} = \mathbf{U}\boldsymbol{\Sigma}^6\mathbf{U}^T = \mathbf{G}^3$ with positive eigenvalues $\lambda_i^3$, giving $\kappa_d(\tilde{\mathbf{G}}) = \kappa_d(\mathbf{G})^3$. The matrix Bernstein inequality yields operator norm deviation $O(\lambda_1^3\sqrt{n\log n/m})$; convergence relative to $\lambda_{\min}^+(\mathbf{K}_\infty) \geq \tfrac{1}{2}\lambda_r^3$ (Schur product inequality, Appendix~\ref{app:proofs}) requires $m \gg \kappa_d(\mathbf{G})^6 n\log n$. See Appendix~\ref{app:proofs} for the full derivation.

\begin{remark}[Practical Width]
\label{rmk:practical_width}
Throughout, \emph{practical width} denotes widths achievable on current hardware ($m \lesssim 10^5$; e.g., GPT-3: $m = 12{,}288$, ViT-Huge: $m = 1{,}280$). The thresholds $m \gg 10^{24}$ (MNIST) and $m \gg 10^{29}$ (CIFAR-10) lie 19 and 24 orders of magnitude beyond this ceiling.
\end{remark}

\begin{remark}[Scope and Generalizations]
\label{rmk:spectral_scope}
Two generalizations follow directly: (1)~stacking $k$ linearized attention layers yields $\tilde{\mathbf{G}} = \mathbf{G}^{2k+1}$ with $\kappa(\tilde{\mathbf{G}}) = \kappa(\mathbf{G})^{2k+1}$; (2)~truncated attention (top-$r$ singular components) caps the condition number at $\lambda_1/\lambda_r$, suggesting low-rank regularization as a mechanism for restoring convergence. Theorem~\ref{thm:cubic_conditioning} bounds NTK deviation \emph{at initialization}; it explains why the network cannot enter the lazy regime but does not directly predict the post-training trajectory (see Discussion for the dataset-dependent behavior). The cubic identity $\tilde{\mathbf{G}}=\mathbf{G}^3$ is construction-specific: the SVD collapse $\mathbf{V}^T\mathbf{V}=\mathbf{I}$ in $\mathbf{X}\mathbf{X}^T\mathbf{X}=\mathbf{U}\boldsymbol{\Sigma}^3\mathbf{V}^T$ requires query, key, and value to share the same $\mathbf{X}$; independent projections (Proposition~\ref{prop:qkv_generalization}) generalize via the same argument on the projected SVD.
\end{remark}

\begin{proposition}[Softmax Spectral Amplification]
\label{prop:softmax_amplification}
Let $\mathbf{G} = \mathbf{U}\boldsymbol{\Lambda}\mathbf{U}^\top$ with eigenvalues $\lambda_1 \geq \cdots \geq \lambda_n > 0$. For unnormalized softmax attention with identity projections, $f^{\mathrm{unnorm}}(\mathbf{X}) = \exp(\mathbf{G}/\sqrt{d})\,\mathbf{X}$, the induced kernel matrix is:
\begin{equation}
\tilde{\mathbf{G}}^{\mathrm{unnorm}} = \mathbf{U}\,\mathrm{diag}\!\left(\lambda_i\,e^{2\lambda_i/\sqrt{d}}\right)\mathbf{U}^\top
\end{equation}
with condition number:
\begin{equation}
\kappa\!\left(\tilde{\mathbf{G}}^{\mathrm{unnorm}}\right) = \kappa(\mathbf{G})\cdot\exp\!\left(\frac{2(\lambda_1-\lambda_n)}{\sqrt{d}}\right)
\end{equation}
This is strictly larger than the linearized bound $\kappa(\mathbf{G})^3$ whenever $\lambda_1 - \lambda_n > \sqrt{d}\,\log\kappa(\mathbf{G})$, a condition satisfied for MNIST and CIFAR-10 at all tested widths.
\end{proposition}

\emph{Proof.} Via eigendecomposition of $\mathbf{G}$; see Appendix~\ref{app:proofs}.

\begin{remark}[Softmax Amplification Is Exponential, Not Polynomial]
\label{rmk:softmax_vs_linear}
Linearized attention yields polynomial amplification $\kappa(\mathbf{G})^3$; unnormalized softmax yields $\kappa(\mathbf{G})\cdot\exp(2(\lambda_1-\lambda_n)/\sqrt{d})$, exponential in the spectral gap, and strictly larger than $\kappa(\mathbf{G})^3$ whenever $\lambda_1 - \lambda_n > \sqrt{d}\,\log\kappa(\mathbf{G})$. The linearized model corresponds to the first-order Taylor truncation $\exp(\mathbf{G}/\sqrt{d}) \approx \mathbf{I} + \mathbf{G}/\sqrt{d}$; higher-order terms only increase spectral amplification. Normalized softmax preserves the spectral gap structurally (rows aligned with $\mathbf{u}_1$ concentrate near one-hot; rows aligned with $\mathbf{u}_n$ remain diffuse), consistent with \citet{bae2022ifsanswer} and \citet{sakai2025nngp}.
\end{remark}

\begin{proposition}[Generalization to Trainable QKV (Query, Key, Value) Projections]
\label{prop:qkv_generalization}
For full-rank $\mathbf{W}_Q, \mathbf{W}_K, \mathbf{W}_V \in \R^{d \times d}$, the bound generalizes to $m = \Omega(\kappa_d(\mathbf{G}_W)^6 n\log n)$ where $\mathbf{G}_W \coloneq \mathbf{X}\mathbf{W}_K\mathbf{W}_K^\top\mathbf{X}^\top$ is the projected Gram matrix (distinct from the NTK matrix $\mathbf{K}$). For orthogonal $\mathbf{W}_K$ with $\mathbf{W}_K^\top\mathbf{W}_K = \frac{d}{2}\mathbf{I}$, $\kappa_d(\mathbf{G}_W) = \kappa_d(\mathbf{G})$ exactly. For He initialization ($\mathbf{W}_K \sim \mathcal{N}(0, 2/d)$), $\E[\mathbf{G}_W] = 2\mathbf{G}$, but $\kappa_d(\mathbf{G}_W)$ depends on the particular draw of $\mathbf{W}_K$; non-convergence holds empirically for MNIST and CIFAR-10 and follows from $\kappa_d(\mathbf{G}_W) \gg 1$ under standard high-dimensional concentration. Non-convergence is not an artifact of identity projections; the orthogonal case establishes this structurally, and the He-init case is supported empirically.
\end{proposition}

% ==============================================================================
% EXPERIMENTS
% ==============================================================================
\section{Experiments}
\label{sec:experiments}

\subsection{Experimental Setup}

\textbf{Datasets.} MNIST and CIFAR-10. Additional results on Fashion-MNIST are provided in the supplementary material.

\textbf{Architectures.} (1) \textbf{2L-ReLU}: Two-layer ReLU network on raw inputs. (2) \textbf{MLP-Attn}: Linearized attention preprocessing ($f^{\text{att}}(\mathbf{X}) = \mathbf{X}\mathbf{X}^T\mathbf{X}$, linearized model) followed by identical MLP. The setup is \emph{transductive}: $\mathbf{X}$ is fixed to the training split at inference; Proposition~\ref{prop:qkv_generalization} partially bridges the inductive gap via learned QKV projections.

\textbf{Initialization.} $\mathbf{w}_r \sim \mathcal{N}(0, \kappa^2 \mathbf{I}_d)$ with $\kappa = 0.01$ (Appendix~\ref{app:experimental}), satisfying the $\|\mathbf{W}\|_F = O(\sqrt{m})$ condition required by the Bernstein bound in Theorem~\ref{thm:cubic_conditioning}'s proof. The width requirement $m = \Omega(\kappa_d(\mathbf{G})^6 n \log n)$ holds for this initialization family (He, Xavier, LeCun) without structural change.

\textbf{Training.} Adam optimizer, $\eta = 10^{-3}$, batch size 128, 500 epochs, $L_2$ regularization $\lambda = 10^{-3}$.

\textbf{NTK Distance.} The distance $\|f_m - f_{\text{NTK}}\|_2$ is measured across widths $m \in \{4, 8, 16, \ldots, 4096\}$ (the exact-computation ceiling; since $\kappa_d(\mathbf{G})^6 \gg m$ for all $m \lesssim 10^5$, the conclusion holds at any practical width---see \S\ref{sec:limitations}); Nystr\"{o}m approximation is avoided because estimation error would conflate with the convergence signal. Figure~\ref{fig:ntk_convergence} is a diagnostic illustration; Theorem~\ref{thm:cubic_conditioning} is a mathematical result that holds independently of the empirical NTK distance values.

\textbf{Perturbation Methods.} Three attacks of increasing strength are used: \textbf{FGSM} (Fast Gradient Sign Method, single-step), \textbf{PGD} (Projected Gradient Descent, $\alpha=0.007$, 40 steps), and \textbf{MIM} (Momentum Iterative Method, PGD with momentum); all at $\varepsilon=0.03$, $L_\infty$, top-10\% influential examples ($\tau=0.1$). Full details in Appendix~\ref{app:experimental}.

\subsection{Key Result: Linearized Attention Cannot Enter the Kernel Regime at Any Practical Width}

Figure~\ref{fig:ntk_convergence} presents the central finding. Table~\ref{tab:ntk_distance} provides detailed measurements.

\begin{table}[t]
\caption{NTK distance $\|f_m - f_{\text{NTK}}\|$ across network widths. 2L-ReLU exhibits monotonically decreasing distance, consistent with convergence to the kernel regime. MLP-Attn shows qualitatively different behavior: non-monotonic across the full width sweep on MNIST (Figure~\ref{fig:ntk_convergence}), monotonically increasing on CIFAR-10.}
\label{tab:ntk_distance}
\vskip 0.1in
\begin{center}
\begin{small}
\begin{tabular}{lccc}
\toprule
\textbf{Model} & \textbf{$m$=16} & \textbf{$m$=1024} & \textbf{$m$=4096} \\
\midrule
\multicolumn{4}{l}{\textit{MNIST}} \\
2L-ReLU & 45.1 & 39.9 & 39.2 \\
MLP-Attn & 10.3 & 33.3 & 43.4 \\
\midrule
\multicolumn{4}{l}{\textit{CIFAR-10}} \\
2L-ReLU & 246.2 & 101.7 & 56.9 \\
MLP-Attn & 3.7 & 10.4 & 12.6 \\
\bottomrule
\end{tabular}
\end{small}
\end{center}
\vskip -0.1in
\end{table}

As shown in Table~\ref{tab:ntk_distance}, 2L-ReLU exhibits monotonically decreasing NTK distance, consistent with standard theory. MLP-Attn shows non-monotonic (MNIST) or increasing (CIFAR-10) distance, consistent with the feature learning regime \citep{chizat2019lazy}. Since $n > d$ for both datasets (MNIST: $n \approx 60{,}000$, $d = 784$; CIFAR-10: $n \approx 50{,}000$, $d = 3{,}072$), the effective condition number $\kappa_d(\mathbf{G})$ (the ratio of the largest to the $d$-th eigenvalue of $\mathbf{G} = \mathbf{X}\mathbf{X}^\top$) governs the spectral amplification (see Remark~\ref{rmk:spectral_scope} for the rank-deficient regime). Measured values $\kappa_d(\mathbf{G}) \approx 1.2 \times 10^3$ (MNIST) and $8.7 \times 10^3$ (CIFAR-10) yield theoretical width requirements $m \gg 10^{24}$ (MNIST) and $m \gg 10^{29}$ (CIFAR-10), far exceeding $m \leq 4096$ tested, in quantitative agreement with Theorem~\ref{thm:cubic_conditioning}.

\textbf{Connection to realistic attention mechanisms.} Softmax and exponential attention (Appendix~\ref{app:softmax_ntk}) exhibit near-zero NTK distance via feature map collapse ($\cos(\phi_i,\phi_j)=1.000$, rank-1 kernel), not genuine convergence, empirically distinguishing three regimes: genuine convergence (2L-ReLU), structured non-convergence (linearized attention), and degenerate collapse (softmax/exp).

\subsection{Influence Malleability Results}

Table~\ref{tab:malleability} shows that MLP-Attn exhibits substantially higher influence malleability across all perturbation types and classification settings.

\begin{table}[t]
\caption{Influence flip rates (\%) across perturbation types ($\varepsilon=0.03$) for both classification settings. In the 10-class setting MLP-Attn shows 6--9$\times$ higher malleability; the gap narrows to ${\approx}\,1\times$ in binary CIFAR-10 (cars vs.\ planes) due to lower $\kappa_d(\mathbf{G})$, consistent with Theorem~\ref{thm:cubic_conditioning} (Appendix~\ref{app:binary_cifar}). Binary setting follows \citet{zhang2022rethinking}. Epoch-by-epoch progression in Table~\ref{tab:training_progression_extended} (Appendix~\ref{app:learning_dynamics}).}
\label{tab:malleability}
\vskip 0.1in
\begin{center}
\begin{small}
\begin{tabular}{lllccc}
\toprule
\textbf{Dataset} & \textbf{Setting} & \textbf{Model} & \textbf{FGSM} & \textbf{PGD} & \textbf{MIM} \\
\midrule
\multirow{4}{*}{MNIST} & \multirow{2}{*}{10-class} & 2L-ReLU & 4.1 & 3.3 & 3.4 \\
& & MLP-Attn & 34.6 & 28.9 & 21.9 \\
\cmidrule(lr){2-6}
& \multirow{2}{*}{Binary} & 2L-ReLU & 8.4 & 8.4 & 8.6 \\
& & MLP-Attn & 25.9 & 41.0 & 40.5 \\
\midrule
\multirow{4}{*}{CIFAR-10} & \multirow{2}{*}{10-class} & 2L-ReLU & 3.3 & 3.1 & 3.2 \\
& & MLP-Attn & 26.4 & 19.1 & 20.5 \\
\cmidrule(lr){2-6}
& \multirow{2}{*}{Binary} & 2L-ReLU & 15.2 & 15.5 & 15.3 \\
& & MLP-Attn & 14.3 & 14.0 & 14.8 \\
\bottomrule
\end{tabular}
\end{small}
\end{center}
\vskip -0.1in
\end{table}

\textbf{Key observations:} (1) In the 10-class setting, MLP-Attn shows 6--9$\times$ higher flip rates than 2L-ReLU across all perturbation types; PGD yields the largest ratio on MNIST ($8.8\times$). (2) In the binary setting, the advantage persists on MNIST ($3$--$4.9\times$) but vanishes on CIFAR-10 (${\approx}\,1\times$), consistent with Theorem~\ref{thm:cubic_conditioning} (Appendix~\ref{app:binary_cifar}).

Figure~\ref{fig:malleability_analysis} visualizes these patterns across perturbation types and intervention strategies.

\begin{figure*}[t]
    \centering
    \begin{subfigure}[b]{0.48\textwidth}
        \centering
        \includegraphics[width=\textwidth]{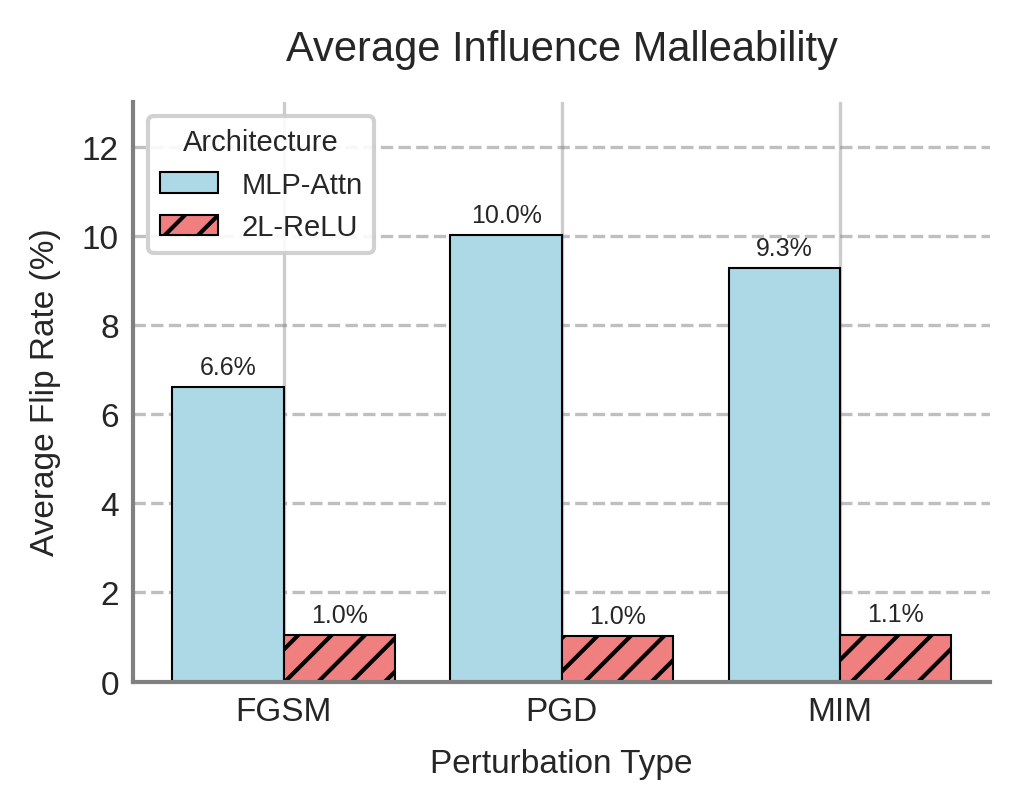}
        \caption{Influence flip rate by perturbation type.}
        \label{fig:flip_rate}
    \end{subfigure}
    \hfill
    \begin{subfigure}[b]{0.48\textwidth}
        \centering
        \includegraphics[width=\textwidth]{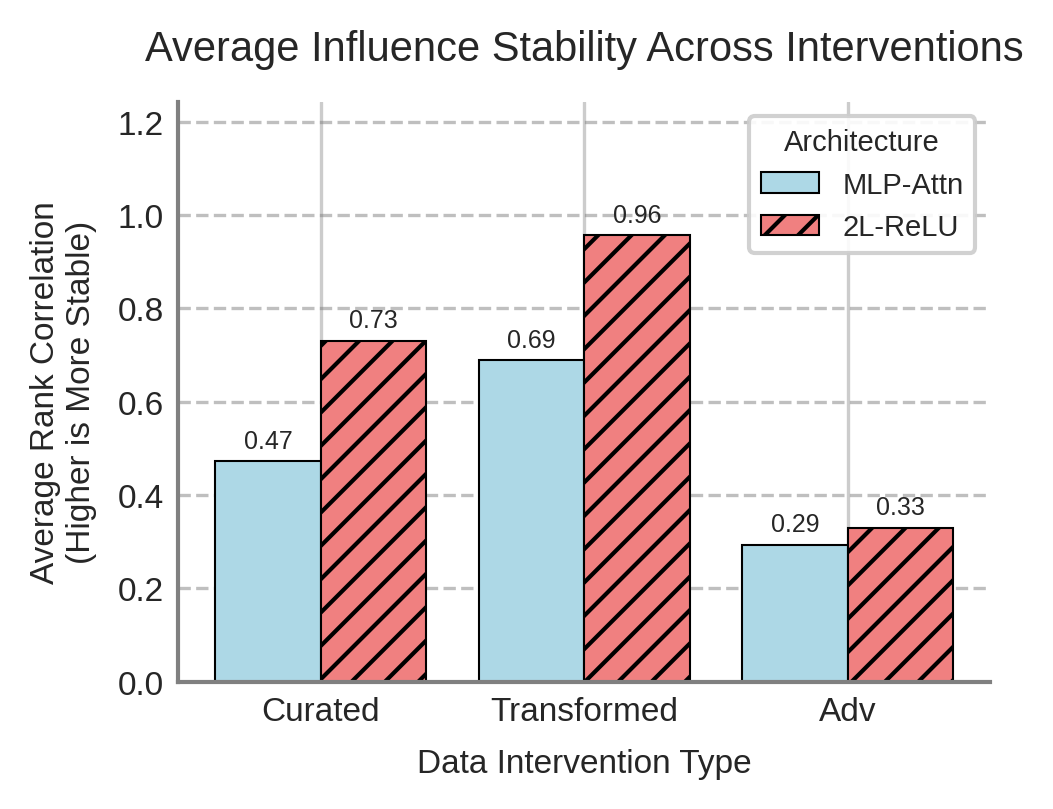}
        \caption{Influence rank correlation by intervention strategy.}
        \label{fig:rank_corr}
    \end{subfigure}
    \caption{Analysis of influence dynamics. (a) MLP-Attn exhibits consistently higher influence malleability (flip rate) across all perturbation types, reflecting its operation in the feature learning regime. (b) Rank correlation analysis reveals that 2L-ReLU maintains rigid influence rankings while MLP-Attn shows lower correlation, indicating continuous re-evaluation of data dependencies. The \emph{Transformed} intervention (replacing influential examples with adversarial versions) produces the most structured adaptation pattern.}
    \label{fig:malleability_analysis}
\end{figure*}

\subsection{Adversarial Training Analysis}

Both architectures are evaluated under PGD adversarial training (PGD-AT, $\varepsilon=0.03$; Table~\ref{tab:adv_training}).

\begin{table}[t]
\caption{Effect of adversarial training on influence flip rates (\%). Full training-epoch progression in Table~\ref{tab:training_progression_extended} (Appendix~\ref{app:learning_dynamics}).}
\label{tab:adv_training}
\vskip 0.1in
\begin{center}
\begin{small}
\begin{tabular}{llcc}
\toprule
\textbf{Dataset} & \textbf{Model} & \textbf{Standard} & \textbf{Adv-Trained} \\
\midrule
\multirow{2}{*}{MNIST} & 2L-ReLU & 3.3 & 43.4 \\
& MLP-Attn & 28.9 & 42.2 \\
\midrule
\multirow{2}{*}{CIFAR-10} & 2L-ReLU & 3.1 & 36.5 \\
& MLP-Attn & 19.1 & 38.6 \\
\bottomrule
\end{tabular}
\end{small}
\end{center}
\vskip -0.1in
\end{table}

\textbf{Key finding:} Adversarial training substantially increases malleability for both architectures (2L-ReLU: $3.3\%\to 43.4\%$ on MNIST), but MLP-Attn achieves comparable malleability under standard training, confirming that attention structurally induces the sensitivity that adversarial training forces upon ReLU networks. This connects directly to NTK non-convergence via Proposition~\ref{prop:kernel_sensitivity}: $K_{\text{LinAttn}}$ sensitivity propagates globally through $\mathbf{G}=\mathbf{X}\mathbf{X}^T$, while 2L-ReLU sensitivity is dataset-independent (Figure~\ref{fig:influencers}, Appendix~\ref{app:experimental}).

% ==============================================================================
% DISCUSSION
% ==============================================================================
\section{Discussion}
\label{sec:discussion}

\textbf{Dual implications of malleability.} The data-dependent kernel $K_{\text{LinAttn}}$ reduces approximation error when targets align with data geometry (Proposition~\ref{prop:bias_reduction}), while the same sensitivity creates a structural malleability gap $O(\sqrt{n}\,\lambda_1(\mathbf{G}))$ (Proposition~\ref{prop:sensitivity_gap}; Appendix~\ref{app:malleability}). This is \emph{bidirectional}: the \emph{Trans} intervention improves accuracy over \emph{Curated} by up to $2.6\%$/$2.3\%$ (MNIST/CIFAR-10, Table~\ref{tab:training_progression_extended}) while enabling adversarial manipulation of influence weights (Table~\ref{tab:adv_training}).

\textbf{Why attention resists the kernel regime.} Linearized attention cubes the condition number via $\tilde{\mathbf{G}}=\mathbf{G}^3$ (Theorem~\ref{thm:cubic_conditioning}), requiring $m \gg 10^{24}$ for MNIST. This aligns with \citet{kim2024nonconvex}, whose mean-field analysis shows MLP-plus-attention Transformers operate in the feature-learning regime; Theorem~\ref{thm:cubic_conditioning} provides the spectral mechanism. Low-$\kappa_d$ configurations may not face the same barrier (cf.\ Table~\ref{tab:malleability}, binary setting), but kernel-regime guarantees do not hold for attention at any practical scale.

% ==============================================================================
% LIMITATIONS
% ==============================================================================
\section{Limitations}
\label{sec:limitations}

\textbf{Linearization and softmax.} The cubic conditioning result uses $f^{\text{att}}(\mathbf{X}) = \mathbf{X}\mathbf{X}^T\mathbf{X}$; Appendix~\ref{app:softmax_ntk} shows softmax and exponential attention instead collapse to rank-1 kernels ($\cos(\phi_i,\phi_j)=1.000$), a distinct failure mode consistent with \citet{bae2022ifsanswer} and \citet{sakai2025nngp}.

\textbf{Scale.} The controlled setting ($m \leq 4096$) enables exact NTK computation. Since $\kappa_d(\mathbf{G})$ is estimable via randomized SVD in $O(nd\log n)$ \citep{halko2011finding}, the bound $m = \Omega(\kappa_d(\mathbf{G})^6 n\log n)$ applies at any scale without gradient computation.

% ==============================================================================
% CONCLUSION
% ==============================================================================
\section{Conclusion}
\label{sec:conclusion}

Linearized attention is provably excluded from the kernel regime at practical widths: cubic conditioning forces $m = \Omega(\kappa_d(\mathbf{G})^6 n \log n)$ beyond feasibility, a barrier confirmed empirically by the diverging NTK distance of MLP-Attn across widths spanning standard transformer architectures, and extending to trainable QKV projections (Proposition~\ref{prop:qkv_generalization}). This non-convergence is the mechanism underlying both the 2--9$\times$ malleability gap and the bidirectional tradeoff between approximation bias reduction and adversarial susceptibility identified in this work; since $\kappa_d(\mathbf{G})$ is estimable in $O(nd\log n)$, the diagnostic applies at any scale (Appendix~\ref{app:practitioner}).

% ==============================================================================
% REFERENCES
% ==============================================================================
\clearpage
\bibliographystyle{unsrtnat}
\bibliography{references}

% ==============================================================================
% APPENDIX
% ==============================================================================
\newpage
\appendix

\section*{Supplementary Material}

This supplementary provides extended proofs, detailed experimental methodology, additional results on Fashion-MNIST, and formal malleability theory supporting the main paper's findings.

\textbf{Notation:} Bold-faced letters represent vectors and matrices, with $[\mathbf{A}]_{ij}$ denoting the $(i,j)$-th entry of matrix $\mathbf{A}$. Throughout, $\mathbf{1} \in \R^n$ denotes the vector of all ones and $\text{vec}(\cdot)$ denotes vectorization. $\boldsymbol{\mathcal{I}} \coloneq (\mathbf{K}+\lambda\mathbf{I})^{-1}\mathbf{y} \in \R^n$ denotes the \emph{influence vector} whose $i$-th entry is the leave-one-out influence weight for training point $i$.

\section{Experimental Setup and Implementation Details}
\label{app:experimental}

\subsection{Architecture Specifications}

The experimental framework compares two primary architectures to isolate the effect of linearized attention:

\begin{enumerate}
\item \textbf{Baseline (2L-ReLU):} Two-layer ReLU network with direct input processing
\item \textbf{Attention-Enhanced (MLP-Attn):} Sequential architecture with linearized attention preprocessing
\end{enumerate}

Both architectures employ identical MLP components with the following specifications:
\begin{itemize}
\item \textbf{Hidden layer width:} $m = 1024$ neurons
\item \textbf{First layer initialization:} $w_r \sim \mathcal{N}(0, \kappa^2 I_d)$ with $\kappa = 0.01$
\item \textbf{Second layer weights:} $a_r \in \{-1, 1\}$
\item \textbf{Activation function:} ReLU
\item \textbf{Output scaling:} $1/\sqrt{m}$ normalization
\end{itemize}

The linearized attention mechanism transforms inputs via $f^{\text{att}}(\mathbf{X}) = \mathbf{X}\mathbf{X}^T\mathbf{X}$ before MLP processing, ensuring that architectural differences reflect attention effects rather than capacity variations. Importantly, $\mathbf{X} \in \R^{n \times d}$ represents the \emph{entire training set}, not mini-batches; the attention transformation is computed once over all training data, making the architecture transductive. This design ensures the data-dependent kernel sensitivity (Proposition~\ref{prop:kernel_sensitivity}) reflects the full dataset correlation structure.

\textbf{Scale normalization:} To ensure that observed differences between MLP-Attn and 2L-ReLU reflect structural properties rather than input scale effects, all input data is normalized to unit $\ell_2$ norm ($\|\mathbf{x}_i\|_2 = 1$) as assumed in Theorem~\ref{thm:polynomial_correspondence}. Additionally, the attention-transformed features $f^{\text{att}}(\mathbf{X})$ are row-normalized before being passed to the MLP, ensuring comparable input magnitudes across architectures.

\subsection{Binary CIFAR-10 Spectral Argument}
\label{app:binary_cifar}

The binary CIFAR-10 subset (cars vs.\ airplanes) yields ${\approx}\,1\times$ malleability ratio, matching 2L-ReLU rather than the 6--9$\times$ gap seen in 10-class CIFAR-10. The mechanism is spectral.

\textbf{Clarification on $\kappa_d(\mathbf{G})$ computation.} The Gram matrix $\mathbf{G} = \mathbf{X}\mathbf{X}^T$ depends only on the input features $\mathbf{X}$, not on labels. \emph{Binary CIFAR-10} here means the binary-class input \emph{subset}: $\mathbf{X}$ consists only of cars and airplanes images (different from the full 10-class $\mathbf{X}$). The condition number $\kappa_d(\mathbf{G})$ is computed on this binary-subset Gram matrix, not the full-dataset matrix. This distinction is important: changing from 10-class to binary relabeling on the same input set would leave $\kappa_d(\mathbf{G})$ unchanged, but selecting a 2-class \emph{input subset} changes $\mathbf{X}$ and hence $\mathbf{G}$.

Cars and airplanes share coarse global structural features (horizontal elongated shapes, vehicle silhouettes, high-frequency backgrounds), which substantially reduces inter-class spectral spread relative to 10 visually distinct classes. This lowers $\kappa_d(\mathbf{G}_{\text{binary}})$ for the binary-subset inputs compared to the full 10-class dataset ($\kappa_d \approx 8.7\times10^3$ for 10-class CIFAR-10). Since cubic amplification scales as $\kappa_d(\mathbf{G})^6$, even a moderate reduction in $\kappa_d$ exponentially compresses the convergence gap between ReLU and attention, bringing both architectures into a regime where the conditioning advantage of ReLU is no longer detectable at $\varepsilon = 0.03$. The Step~1 practitioner diagnostic (Appendix~\ref{app:practitioner}) should therefore be applied to the specific binary-class input subset rather than the full dataset $\kappa_d$.

\subsection{Dataset Configuration and Preprocessing}

The experimental evaluation employs three standard computer vision datasets with specific preprocessing protocols, following established benchmarking practices \citep{krizhevsky2009learning}:

\subsubsection{Dataset Specifications}
\begin{table}[h]
\caption{Dataset specifications. Dim.\ = input dimension, Norm.\ = (mean, std) normalization parameters. Standard training set sizes are used.}
\centering
\footnotesize
\begin{tabular}{@{}lccc@{}}
\toprule
\textbf{Dataset} & \textbf{Dim.} & \textbf{Classes} & \textbf{Norm.} \\
\midrule
MNIST & 784 & 10 & (.13, .31) \\
Fashion-MNIST & 784 & 10 & (.29, .35) \\
CIFAR-10 & 3072 & 10 & (.5, .5, .5) \\
\bottomrule
\end{tabular}
\end{table}

\subsubsection{Class Selection and Data Sampling}
The experimental design considers two classification scenarios:
\begin{itemize}
\item \textbf{Binary Classification:} Select representative class pairs (e.g., MNIST digits 3 vs 8, CIFAR-10 cars vs planes)
\item \textbf{Multi-class Classification:} All ten (10) classes.
\end{itemize}
\textbf{Main results} (Tables~\ref{tab:ntk_distance}--\ref{tab:malleability} and~\ref{tab:adv_training}, Figure~\ref{fig:ntk_convergence}) use the \textbf{10-class setting}. Table~\ref{tab:malleability} (main paper) reports both 10-class and binary classification results; binary setting follows \citet{zhang2022rethinking}. Epoch-by-epoch training progression across all datasets and intervention types is provided in Table~\ref{tab:training_progression_extended} (Appendix~\ref{app:learning_dynamics}).

\subsection{Training Configuration}

\subsubsection{Optimization Parameters}
\begin{itemize}
\item \textbf{Optimizer:} Adam with learning rate $\eta = 10^{-3}$
\item \textbf{Batch size:} 128 samples
\item \textbf{Training epochs:} 500
\item \textbf{Regularization:} $\lambda = 10^{-3}$ for kernel ridge regression
\item \textbf{Mixed precision:} Enabled using PyTorch's GradScaler for memory efficiency
\end{itemize}

\subsubsection{Loss Function and Regularization}
The training objective incorporates both task loss and parameter regularization:
$$\mathcal{L}(\boldsymbol{\theta}) = \frac{1}{2n} \sum_{i=1}^n \ell(f(\mathbf{x}_i; \boldsymbol{\theta}), y_i) + \frac{\lambda}{2} \|\boldsymbol{\theta} - \boldsymbol{\theta}_0\|^2$$

where:
\begin{itemize}
\item $\ell(\cdot, \cdot)$ is cross-entropy loss for multi-class or MSE loss for binary classification
\item $\boldsymbol{\theta}_0$ represents initial parameters
\item $\lambda = 10^{-3}$ provides regularization strength
\end{itemize}

\subsection{Evaluation Metrics and Analysis}

\subsubsection{Influence Function Computation}
Following the NTK-based approach of \citet{zhang2022rethinking} and the foundational influence function methodology of \citet{koh2017understanding}, influence is computed using the \emph{empirical finite-width NTK} $\mathbf{K}_m$ at each width $m$, not the theoretical infinite-width limit. Specifically, for each trained model, the NTK is computed as $[\mathbf{K}_m]_{ij} = \langle \nabla_{\boldsymbol{\theta}} f(\mathbf{x}_i), \nabla_{\boldsymbol{\theta}} f(\mathbf{x}_j) \rangle$ using the actual model gradients. This approach ensures that (1) influence comparisons between architectures are fair (both evaluated at the same width), and (2) the influence patterns reflect the actual model dynamics rather than a theoretical approximation. The key finding that MLP-Attn diverges from its infinite-width limit does not invalidate this methodology; rather, the empirical NTK captures how each architecture actually represents training data relationships at practical widths.

\subsubsection{Malleability Metrics}
The analysis employs several quantitative measures:

\begin{enumerate}
\item \textbf{Influence Flip Rate:}
$$\text{Flip Rate} = \frac{1}{|\mathcal{H}|} \sum_{i \in \mathcal{H}} \mathbf{1}[\text{sign}(I_i^{\text{orig}}) \neq \text{sign}(I_i^{\text{adv}})]$$
where $\mathcal{H}$ denotes the set of high-influence training examples (top 10\% positive influence).

\item \textbf{Rank Correlation (Spearman's $\rho$):}
$$\rho = 1 - \frac{6 \sum_{i=1}^n d_i^2}{n(n^2 - 1)}$$
where $d_i$ is the difference between ranks of influence scores before and after perturbation.

\item \textbf{Top-K Stability:}
$$\text{Stability}_K = \frac{|\mathcal{T}_K^{\text{orig}} \cap \mathcal{T}_K^{\text{adv}}|}{K}$$
where $\mathcal{T}_K^{\text{orig}}$ and $\mathcal{T}_K^{\text{adv}}$ are the sets of top-K influential examples before and after perturbation.
\end{enumerate}

\subsection{Empirical Analysis of Influence Stability and Malleability}

To empirically validate the theoretical framework, this section presents a comprehensive analysis comparing the baseline \textbf{2L-ReLU} network against the attention-enhanced \textbf{MLP-Attn} architecture. This investigation leverages the metrics defined above to explore two core phenomena: the stability of influence rankings under various data interventions and the malleability of influence signs under targeted adversarial attacks.

\subsubsection{Influence Stability Across Data Interventions}

The analysis first assesses influence stability using \textbf{Rank Correlation (Spearman's $\rho$)}. A higher correlation indicates a more rigid influence structure, while a lower correlation suggests a more flexible, adaptive re-ranking of training point importance.

The stability analysis reveals two distinct patterns of behavior:
\begin{itemize}
    \item \textbf{Constant Stability for Curated and Transformed Interventions:} Stability for both architectures is remarkably constant throughout training across both classification settings (binary and 10-class). This indicates that the relative importance of training examples is a fixed structural property for these scenarios. In nearly all cases, the 2L-ReLU architecture exhibits a significantly higher Spearman's $\rho$, signifying a more rigid influence structure compared to the MLP-Attn model.

    \item \textbf{Dynamic Stability for Adversarial Interventions:} The `Adv' intervention induces dynamic stability profiles that evolve over training. For the more complex Fashion-MNIST and CIFAR-10 datasets, both architectures generally show a decline in stability. On the simpler MNIST dataset, the 2L-ReLU model shows a unique ability to recover and increase its influence stability over time, a phenomenon not observed in the attention-enhanced model.
\end{itemize}

\subsubsection{Influence Malleability Under Adversarial Attacks}

Influence malleability is evaluated using the \textbf{Influence Flip Rate}, which measures the percentage of high-influence examples that reverse their helpful/harmful status after being perturbed.

The malleability results provide several key insights:
\begin{itemize}
    \item \textbf{PGD Reveals Clearer Patterns:} The stronger PGD attack generally reveals more consistent trends. On Fashion-MNIST, the MLP-Attn model is demonstrably more malleable than the 2L-ReLU baseline. On the more complex CIFAR-10 dataset, the results are mixed, with relative malleability depending on the number of classes.

    \item \textbf{FGSM Reveals Volatility:} The weaker FGSM attack produces more erratic results than PGD. On Fashion-MNIST, flip rates for both models fluctuate across training; on CIFAR-10, the binary vs.\ 10-class contrast is consistent with the spectral argument (Appendix~\ref{app:binary_cifar}).
\end{itemize}

\subsubsection{Synthesis and Insights}

Taken together, these empirical results confirm the paper's central thesis regarding NTK non-convergence and its dual implications. The \textbf{2L-ReLU} architecture, which converges toward its infinite-width NTK limit, is characterized by \textit{high stability} but \textit{low malleability}; its dependence on training data is strong and rigid, consistent with kernel regime behavior. In contrast, the \textbf{MLP-Attn} architecture, which remains in the feature learning regime, exhibits \textit{lower stability} but \textit{higher malleability}.

This trade-off carries dual implications. On one hand, the architecture continuously re-evaluates data dependencies, allowing flexible data relationships. On the other hand, this same flexibility makes it more susceptible to adversarial manipulation of training data, as quantified by the elevated Influence Flip Rate.

\subsection{Reproducibility and Validation}

The experimental framework ensures reproducibility through:

\begin{enumerate}
\item \textbf{Deterministic Initialization:} Fixed random seeds (42) for PyTorch, NumPy, and Python random modules with deterministic CUDA operations.

\item \textbf{Consistent Data Processing:} All datasets undergo identical preprocessing pipelines with fixed normalization parameters and feature scaling.

\item \textbf{Hardware:} Experiments conducted on NVIDIA A100 (10-class) and NVIDIA T4 GPUs (binary classification) using PyTorch with CUDA acceleration and mixed-precision training.
\end{enumerate}

\subsection{Model Complexity Analysis}

Figure~\ref{fig:complexity_baselines} presents the relationship between influence scores and model complexity contributions. The training data is partitioned into 10 groups based on influence scores, ordered from most harmful (group 1) to most helpful (group 10).

\begin{figure*}[h]
    \centering
    \begin{subfigure}[b]{0.32\textwidth}
        \centering
        \includegraphics[width=\textwidth]{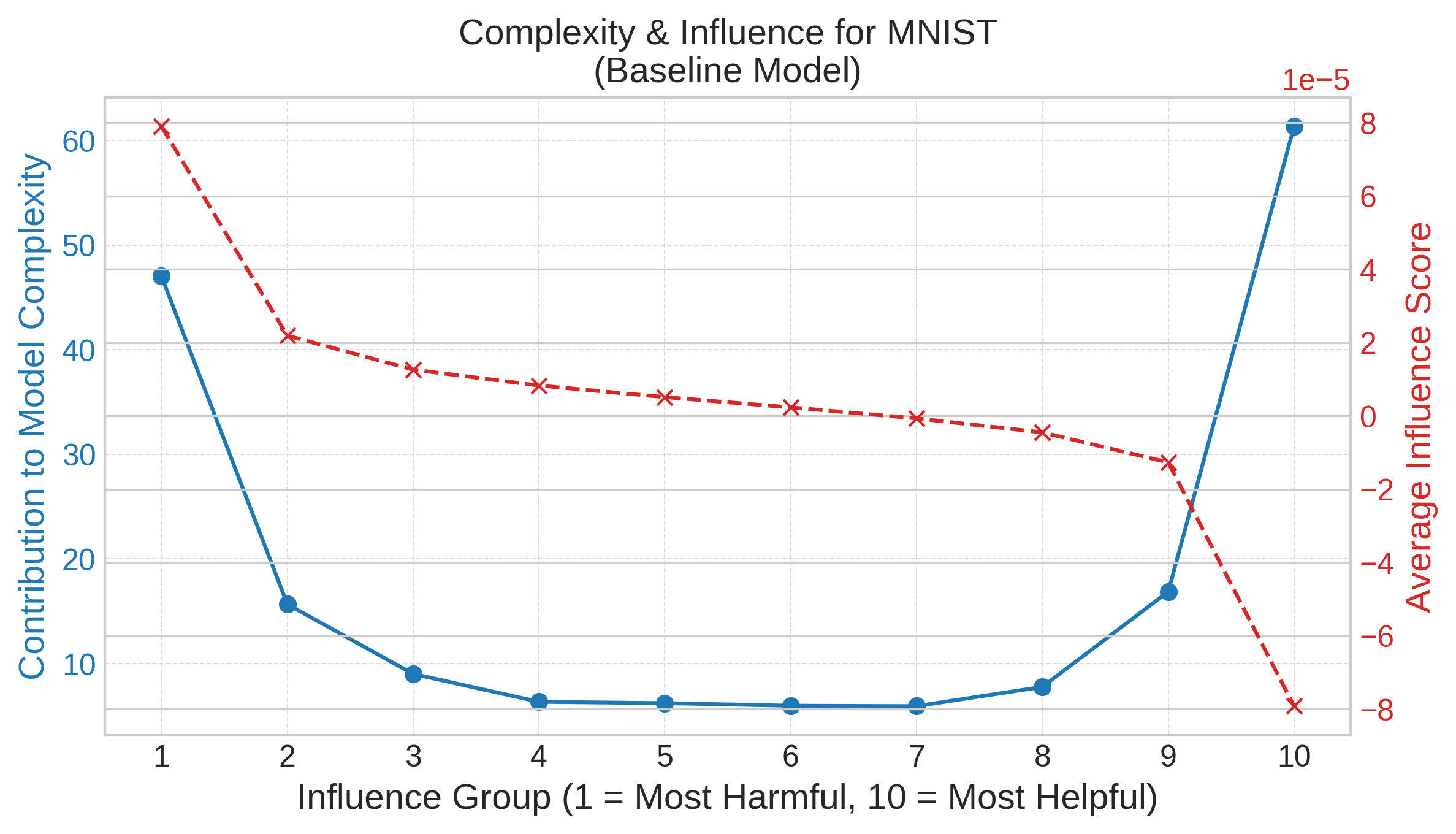}
        \caption{MNIST}
    \end{subfigure}
    \hfill
    \begin{subfigure}[b]{0.32\textwidth}
        \centering
        \includegraphics[width=\textwidth]{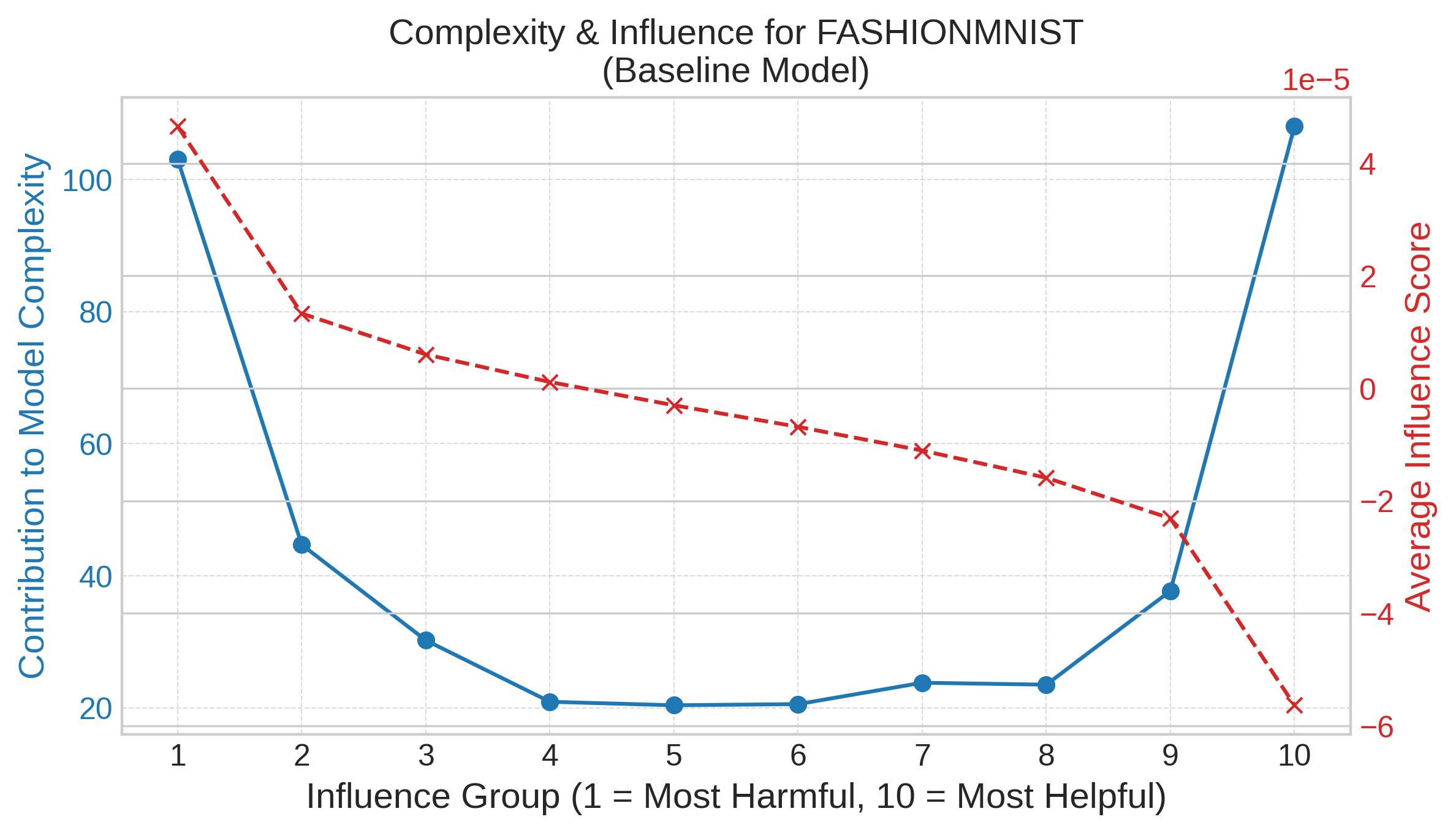}
        \caption{Fashion-MNIST}
    \end{subfigure}
    \hfill
    \begin{subfigure}[b]{0.32\textwidth}
        \centering
        \includegraphics[width=\textwidth]{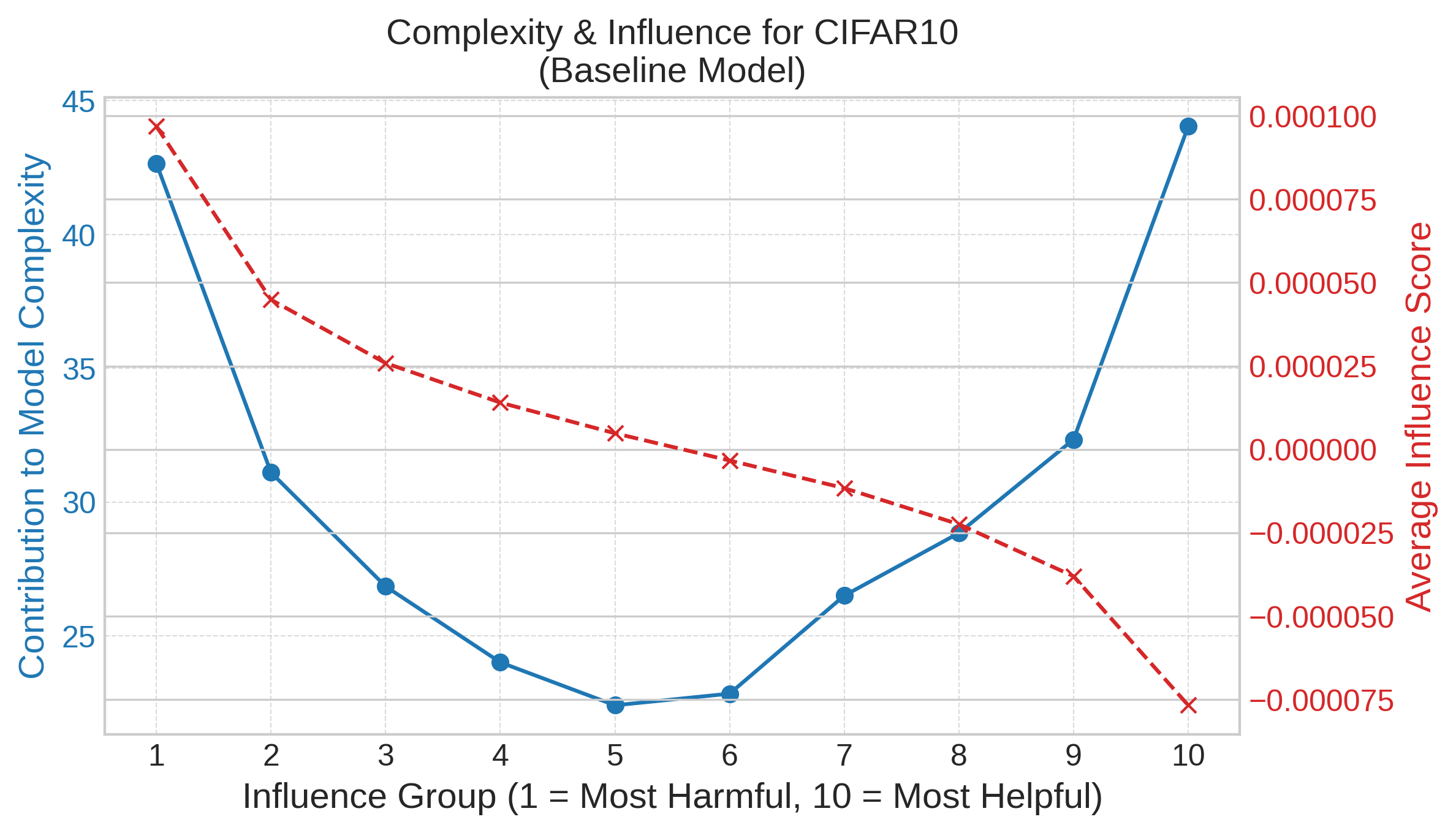}
        \caption{CIFAR-10}
    \end{subfigure}
    \caption{Contribution to model complexity (blue) and average influence score (red) for MLP-Attn across three datasets. The U-shaped complexity curve indicates that the most influential points (both harmful and helpful) contribute most to model complexity, consistent with findings of \citet{zhang2022rethinking}.}
    \label{fig:complexity_baselines}
\end{figure*}

\textbf{Key observation:} The U-shaped complexity curve is consistent across all three datasets, confirming that influence malleability concentrates on the training examples that matter most for model complexity. This connects the malleability metric to the RKHS complexity of the learned function: high-malleability architectures like MLP-Attn are most sensitive precisely at the training points that contribute most to model capacity.

\subsection{Numerical Stability Verification}

To ensure reliable influence computations, the implementation verifies stability at each experimental configuration:
\begin{enumerate}
\item \textbf{Positive definiteness:} Assert $\lambda_{\min}(\mathbf{K}) > -\tau$ with $\tau = 10^{-12}$
\item \textbf{Conditioning:} Assert $\kappa(\mathbf{K} + \lambda\mathbf{I}) < 10^{12}$
\item \textbf{Inversion accuracy:} Verify $\|(\mathbf{K} + \lambda\mathbf{I})^{-1}(\mathbf{K} + \lambda\mathbf{I}) - \mathbf{I}\|_F < \tau$
\item \textbf{Symmetry:} Assert $\|\mathbf{K} - \mathbf{K}^T\|_F < \tau$
\end{enumerate}
All experimental configurations passed these checks, confirming that the NTK-based influence computations are numerically reliable.

\subsection{Influential Training Examples}

\begin{figure}[h]
    \centering
    \includegraphics[width=\columnwidth]{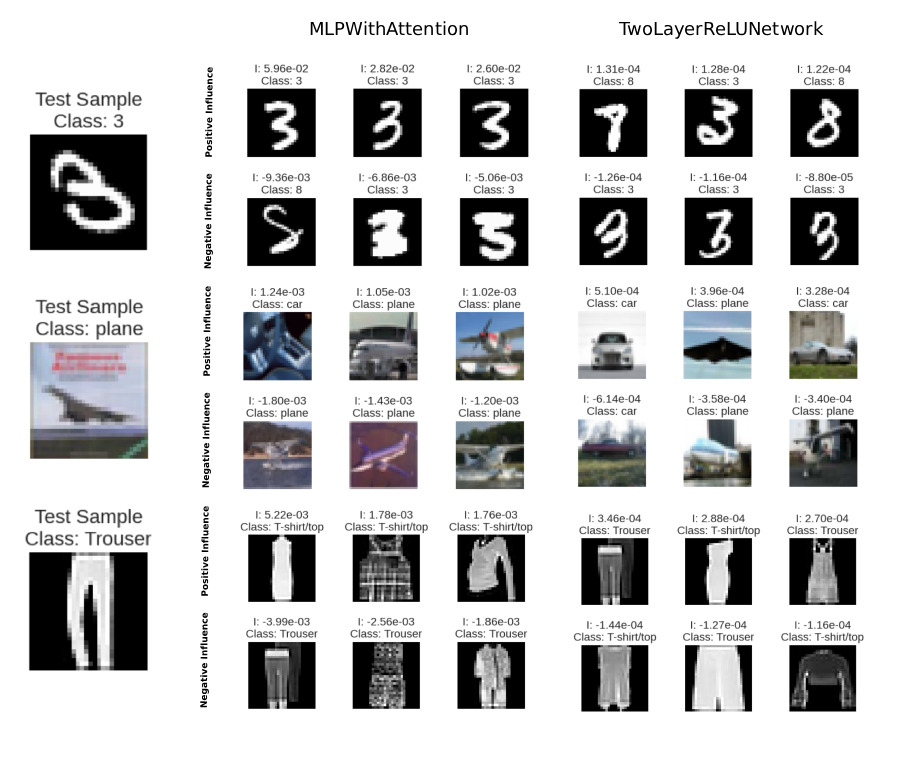}
    \caption{Top influential training examples for a representative test digit. Positive influencers (top rows) increase test loss when removed; negative influencers (bottom rows) decrease it. The key distinction is quantitative: MLP-Attn flip rate 28.9\% vs.\ 3.3\% for 2L-ReLU under PGD (Table~\ref{tab:malleability}), reflecting the Gram-induced kernel's data-dependent structure.}
    \label{fig:influencers}
\end{figure}

\subsection{Loss Landscape Analysis}

Figure~\ref{fig:landscape_3d} presents the 3D loss landscape for both architectures at width $m = 1024$ on MNIST, visualized along two random directions in parameter space.

\begin{figure*}[h]
    \centering
    \includegraphics[width=\textwidth]{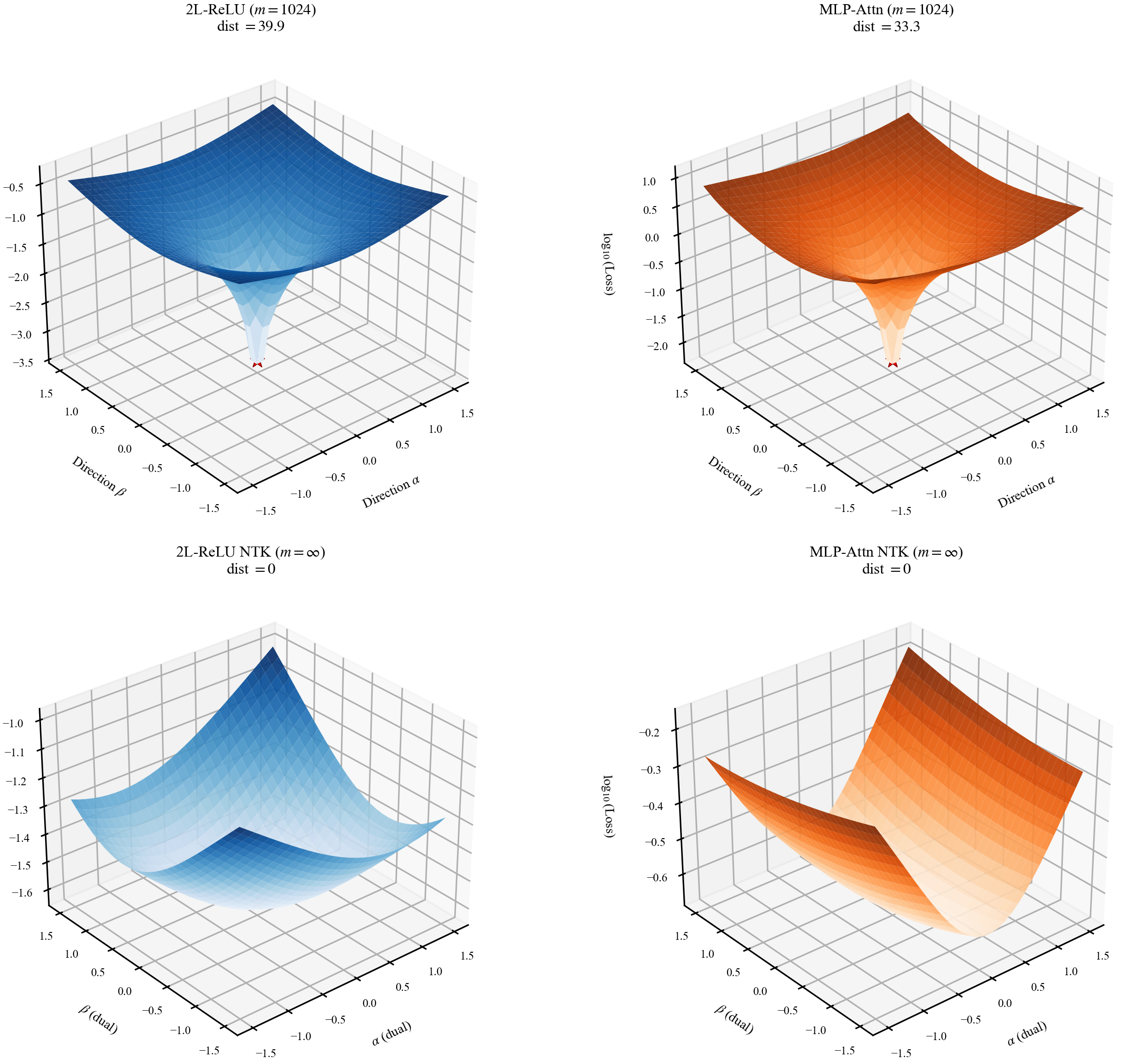}
    \caption{Loss landscape comparison (MNIST) at finite width $m = 1024$ (left two panels) and infinite-width NTK limit (right two panels). \textbf{2L-ReLU} (blue) converges toward its NTK landscape: both finite and infinite-width surfaces share similar geometry. \textbf{MLP-Attn} (orange) shows a qualitatively different finite-width landscape (sharp, deep minimum) compared to its NTK limit (broad, shallow basin), visualizing the NTK non-convergence from a loss geometry perspective. NTK distances: 2L-ReLU $= 39.93$, MLP-Attn $= 33.30$.}
    \label{fig:landscape_3d}
\end{figure*}

The landscape geometry connects directly to the dual implications framework: the broader minimum of MLP-Attn reflects the feature learning regime (NTK non-convergence), where the model continuously adapts its representation. This geometric flexibility underlies both the higher influence malleability and the sensitivity to training data perturbations.

\subsection{Practitioner Guidance for Influence Methods on Attention Architectures}
\label{app:practitioner}

The following steps operationalize the non-convergence result for practitioners deploying TRAK \citep{park2023trak}, DataInf \citep{kwon2023datainf}, or EK-FAC \citep{grosse2023ekfac} on attention-based models.

\emph{Step 1 --- Diagnostic.} Compute $\kappa_d(\mathbf{G})$ via randomized SVD \citep{halko2011finding} in $O(nd\log n)$ time. This requires no gradient computation and no training. Values exceeding $10^2$ place the dataset in the non-convergence regime: MNIST yields $\kappa_d \approx 10^3$, CIFAR-10 yields $\kappa_d \approx 8.7\times 10^3$. In both cases, $m = \Omega(\kappa_d^6 n\log n)$ exceeds $10^{24}$ (MNIST) and $10^{29}$ (CIFAR-10), so no practical width satisfies the kernel regime assumption required by these tools' theoretical guarantees. To put these thresholds in perspective: storing $10^{24}$ neurons at single precision requires ${\sim}4\times 10^6$ exabytes, twelve orders of magnitude beyond the largest known architectures (${\sim}10^{12}$ parameters), and the CIFAR-10 bound is further elevated by the higher $\kappa_d(\mathbf{G})^6$ scaling of that dataset.

\emph{Step 2 --- Data auditing and poisoning detection.} The 6--9$\times$ higher flip rates imply that influence scores for attention are less stable than for ReLU networks. Recommended adjustments: (a) lower the confidence threshold for labeling examples \emph{harmful} or \emph{helpful}; (b) ensemble influence scores across multiple training checkpoints rather than relying on a single checkpoint; (c) treat borderline high-influence examples with increased scrutiny before removal.

\emph{Step 3 --- Active learning.} Influence-based active learning pipelines (BADGE, coreset methods) assume stable rankings. For attention at practical widths, rankings are structurally unstable as a consequence of Theorem~\ref{thm:cubic_conditioning}. Re-rank after each training phase rather than computing influence rankings once at initialization.

\emph{When is the effect attenuated?} In binary classification with homogeneous classes (e.g., cars vs.\ airplanes), inter-class spectral spread is lower, yielding a smaller $\kappa_d(\mathbf{G})$ and reducing the cubic amplification. The Step~1 diagnostic quantifies this directly and is dataset-specific.

\emph{Adversarial training equivalence.} Adversarial training raises 2L-ReLU malleability from $3.3\%$ to $43.4\%$ --- matching attention's structural level. Practitioners combining attention with adversarial training face compounded instability and should validate influence scores across multiple perturbation budgets $\epsilon \in \{0.1, 0.2, 0.3, 0.5\}$.

\subsection{Extended NTK Validation: Softmax and Exponential Attention}
\label{app:softmax_ntk}

To contextualize the non-convergence of linearized attention, two additional attention variants are evaluated: \textbf{softmax attention} and \textbf{exponential attention}. Both use the same transductive setup as the main experiment---the reference set is fixed to the full training split before any width sweep.

\textbf{Softmax attention} applies row-wise normalization to the score matrix:
\begin{equation}
\phi_{\text{sm}}(\mathbf{x}) = \ell_2\text{-norm}\!\left(\operatorname{softmax}\!\left(\frac{\hat{\mathbf{x}}^\top \hat{\mathbf{X}}}{\sqrt{d}}\right)\hat{\mathbf{X}}\right),
\end{equation}
where $\hat{\mathbf{x}} = \mathbf{x}/\|\mathbf{x}\|_2$ and $\hat{\mathbf{X}}$ is the row-normalized training matrix.

\textbf{Exponential attention} omits the normalizing denominator, retaining only the unnormalized exponential weights:
\begin{equation}
\phi_{\text{exp}}(\mathbf{x}) = \ell_2\text{-norm}\!\left(\exp\!\left(\frac{\hat{\mathbf{x}}^\top \hat{\mathbf{X}}}{\sqrt{d}} - c(\mathbf{x})\right)\hat{\mathbf{X}}\right),
\end{equation}
where $c(\mathbf{x}) = \max_j \hat{\mathbf{x}}^\top \hat{\mathbf{x}}_j / \sqrt{d}$ is a per-row stability shift that cancels exactly under $\ell_2$-normalization. Unlike softmax, the unnormalized exponential amplifies score differences, so the aggregation is not a simple weighted mean. This corresponds to the Nadaraya-Watson estimator with exponential kernel and no bandwidth normalization.

Figure~\ref{fig:softmax_ntk} shows the NTK distance $1 - \cos(K_m, K_\infty)$ across widths. Both softmax and exponential attention exhibit near-zero distance at all widths, which at first glance appears to indicate rapid convergence. Figure~\ref{fig:phi_diversity} reveals this to be a measurement artifact: the feature maps $\phi_{\text{sm}}$ and $\phi_{\text{exp}}$ collapse to a single direction for all inputs (mean pairwise cosine similarity $= 1.000$ on both MNIST and CIFAR-10). When all $\phi(\mathbf{x}_i)$ are identical, both $K_m$ and $K_\infty$ degenerate to rank-1 matrices and their cosine similarity is trivially $1$, independent of width.

The collapse mechanism for softmax attention is well understood: at standard temperature $T = \sqrt{d}$, all scores $\hat{\mathbf{x}}^\top \hat{\mathbf{x}}_j / \sqrt{d}$ fall in a narrow range for high-dimensional inputs ($d = 784$ for MNIST, $d = 3072$ for CIFAR-10), making the softmax distribution nearly uniform. The aggregated feature $\phi_{\text{sm}}(\mathbf{x}) \approx \ell_2\text{-norm}(\bar{\mathbf{x}})$ for all $\mathbf{x}$, where $\bar{\mathbf{x}}$ is the training set mean. Exponential attention avoids this via unnormalized weights, but the large-$d$ concentration of inner products similarly causes exponential weights to concentrate near uniform, yielding the same collapse.

In contrast, linearized attention and 2L-ReLU maintain diverse feature maps (mean pairwise cosine similarity $0.754$ and $0.318$ on MNIST; $0.066$ and $0.044$ on CIFAR-10), confirming that the non-trivial NTK distances measured for these architectures reflect genuine structural differences between $K_m$ and $K_\infty$---not a measurement artifact.

\textbf{Two metrics, two facets of non-convergence.} The NTK distance $1 - \cos(K_m, K_\infty)$ used in Figure~\ref{fig:softmax_ntk} is \emph{scale-invariant}: it treats $K_m$ and $K_\infty$ as unit vectors of their entries and measures only the angle between them, i.e., eigenvector alignment. In contrast, the function-space distance $\|f_m - f_{\mathrm{NTK}}\|$ used in Figure~1 is \emph{magnitude-sensitive}: it captures both eigenvector misalignment and eigenvalue amplification. For linearized attention, Theorem~\ref{thm:cubic_conditioning} establishes that the effective condition number scales as $\kappa_d(G)^3$, meaning eigenvalues are cubically amplified relative to the reference kernel. At large widths, the \emph{eigenvectors} of $K_m$ may align with those of $K_\infty$ (low cosine distance), while the \emph{eigenvalues} remain severely amplified (high function-space distance). The two figures are therefore not contradictory: Figure~\ref{fig:softmax_ntk} isolates structural misalignment, while Figure~1 captures total deviation including spectral amplification. This localization sharpens the non-convergence result: the failure is specifically in the spectral scaling imposed by the cubic Gram structure, not in the directional geometry of the kernel.

\begin{remark}[Reconciling Proposition~\ref{prop:softmax_amplification} with Empirical Collapse]
\label{rmk:softmax_reconcile}
Proposition~\ref{prop:softmax_amplification} analyzes \emph{unnormalized} softmax $f^{\mathrm{unnorm}}(\mathbf{X}) = \exp(\mathbf{G}/\sqrt{d})\,\mathbf{X}$, yielding a kernel with large but finite condition number. The empirical results above use \emph{row-normalized} softmax (standard attention), where each row of the attention matrix is divided by its sum. Row-wise normalization concentrates attention weights: at temperature $T = \sqrt{d}$, high-dimensional inputs cause all scores to cluster near uniform, making $\phi_{\mathrm{sm}}(\mathbf{x}) \approx \ell_2$-norm($\bar{\mathbf{x}}$) for all $\mathbf{x}$ and collapsing the kernel to rank 1. The two analyses are therefore not contradictory: unnormalized softmax produces a large-$\kappa$ regime (Proposition~\ref{prop:softmax_amplification}), while normalized softmax undergoes feature collapse to $\kappa = \infty$ (rank-1 kernel, $\lambda_2 = 0$), a qualitatively distinct failure mode. Both escape the NTK kernel regime, through different mechanisms. Genuine convergence to the kernel regime requires the limiting kernel $K_\infty$ to be positive definite: a rank-1 kernel contains no information about input geometry and supports no nontrivial prediction. The near-zero NTK distance observed for softmax and exponential attention is therefore a measurement artifact: both $K_m$ and $K_\infty$ degenerate to the same rank-1 matrix independently of width, making cosine similarity trivially unity regardless of whether any convergence has occurred. Collapse is not a path into the kernel regime; it is the elimination of the model's capacity to operate in any principled regime.
\end{remark}

\begin{figure*}[h]
\centering
\includegraphics[width=\textwidth]{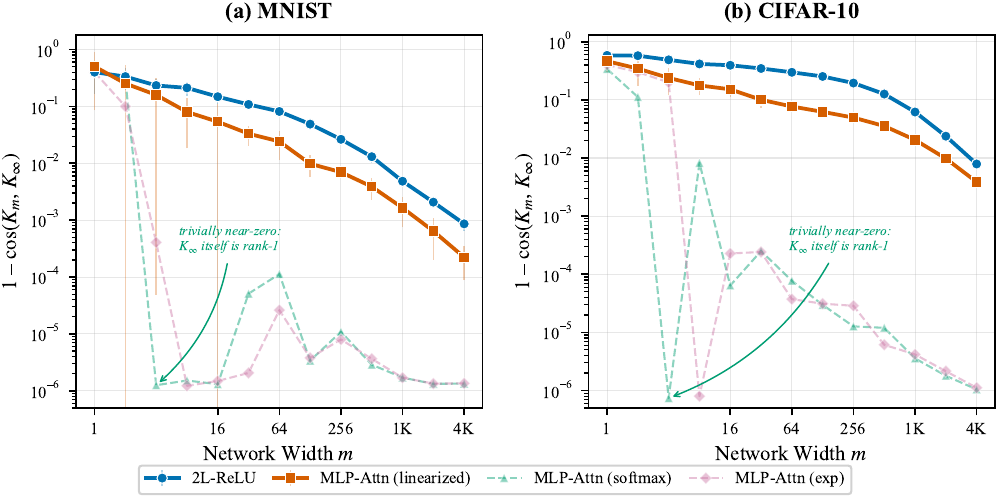}
\caption{NTK distance $1 - \cos(K_m, K_\infty)$ versus network width $m$ for four attention variants (mean $\pm$ std over 10 seeds). \emph{This metric is scale-invariant}: it measures angular alignment (eigenvector structure) between $K_m$ and $K_\infty$, and is blind to eigenvalue magnitudes. Compare with Figure~1, which uses the magnitude-sensitive function-space distance $\|f_m - f_{\mathrm{NTK}}\|$ and shows non-monotonic or increasing divergence for linearized attention. The two are not contradictory: at large widths the eigenvectors of $K_m$ align with those of $K_\infty$ (low cosine distance), while the eigenvalues remain cubically amplified by $\kappa_d(G)^3$ (high function-space distance, cf.\ Theorem~\ref{thm:cubic_conditioning}). \textbf{2L-ReLU and linearized attention} decrease gradually with width, indicating genuine structural alignment; the sweep extends to $m = 4096$, covering and exceeding the hidden dimensions of standard transformers (BERT: 768, ViT-L: 1024, GPT-2: 1600). \textbf{Softmax and exponential attention} show near-zero distance at all widths, but this is vacuous: their feature maps collapse to a single direction (Figure~\ref{fig:phi_diversity}, mean pairwise cosim $= 1.000$), so $K_\infty$ itself is rank-1. Both $K_m$ and $K_\infty$ are identical degenerate matrices at every width, making cosine similarity trivially unity regardless of convergence. The oscillations of the dashed lines near the numerical floor are floating-point noise in this degenerate regime.}
\label{fig:softmax_ntk}
\end{figure*}

\begin{figure*}[h]
\centering
\includegraphics[width=\textwidth]{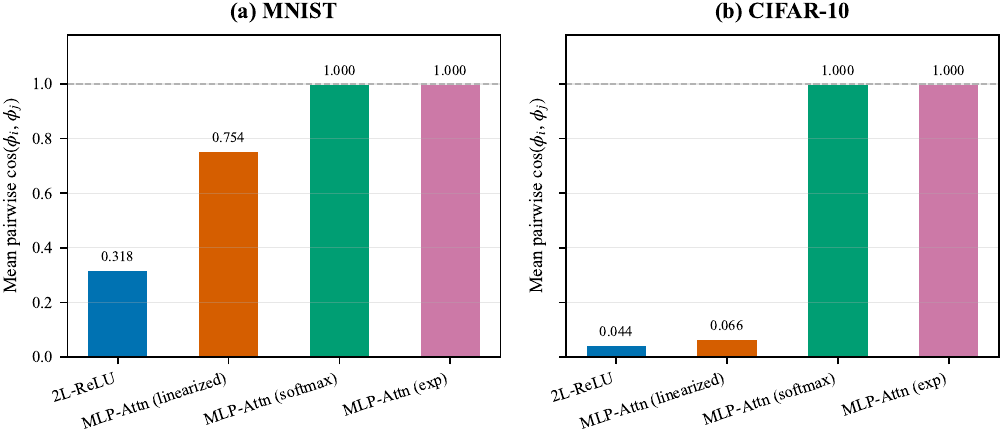}
\caption{Feature map diversity for each attention variant, measured as mean off-diagonal pairwise cosine similarity of $\phi(\mathbf{x}_i)$ embeddings. Values near 1 indicate collapse to a single direction (degenerate). Softmax and exponential attention achieve $\cos = 1.000$ on both datasets, confirming that their near-zero NTK distances in Figure~\ref{fig:softmax_ntk} are artifacts of rank-1 kernel collapse rather than principled convergence. Linearized attention and 2L-ReLU maintain diverse embeddings, validating the non-convergence result of Theorem~\ref{thm:cubic_conditioning}.}
\label{fig:phi_diversity}
\end{figure*}

\newpage
\section{Theoretical Proofs and Derivations}
\label{app:proofs}

\subsection{Linearized Attention Approximation}

\begin{theorem}[Taylor Approximation of Softmax Attention]
The linearized attention mechanism $f^{\text{att}}(\mathbf{X}) = \mathbf{X}\mathbf{X}^T\mathbf{X}$ provides a first-order Taylor approximation to standard softmax attention with bounded approximation error.
\end{theorem}

\begin{proof}
Consider standard softmax attention \citep{vaswani2017attention}:
$\text{Attention}(\mathbf{X}) = \text{softmax}\left(\frac{\mathbf{X}\mathbf{X}^T}{\sqrt{d}}\right) \mathbf{X}$

For row $i$, the softmax output is $[\text{softmax}(\mathbf{A})]_{ij} = \exp(A_{ij})/\sum_k \exp(A_{ik})$ where $\mathbf{A} = \mathbf{X}\mathbf{X}^T/\sqrt{d}$. For small entries $|A_{ij}| \ll 1$, using $\exp(x) \approx 1 + x + O(x^2)$:
\begin{align}
[\text{softmax}(\mathbf{A})]_{ij} &= \frac{1 + A_{ij} + O(A_{ij}^2)}{n + \sum_k A_{ik} + O(\|\mathbf{A}\|^2)} \nonumber \\
&= \frac{1}{n}\left(1 + A_{ij} - \bar{A}_i + O(\|\mathbf{A}\|^2)\right)
\end{align}
where $\bar{A}_i = \frac{1}{n}\sum_k A_{ik}$ is the row mean. The attention output for row $i$ is:
\begin{equation}
[\text{Attention}(\mathbf{X})]_i = \frac{1}{n}\sum_j (1 + A_{ij} - \bar{A}_i)\mathbf{x}_j + O(\|\mathbf{A}\|^2)
\end{equation}
The uniform term $\frac{1}{n}\sum_j \mathbf{x}_j$ contributes only the data mean. The data-dependent term is $\frac{1}{n}\sum_j A_{ij}\mathbf{x}_j = \frac{1}{n\sqrt{d}}[\mathbf{X}\mathbf{X}^T\mathbf{X}]_i$. Absorbing constants and dropping the mean-centering term (which does not affect relative similarities), the linearized attention $f^{\text{att}}(\mathbf{X}) = \mathbf{X}\mathbf{X}^T\mathbf{X}$ captures the essential data-dependent structure with approximation error $O(\|\mathbf{X}\mathbf{X}^T\|^2/d)$.
\end{proof}

\subsection{Exact Gram-Induced Kernel Characterization}

\begin{theorem}[Parameter-Free Gram-Induced Structure]
The parameter-free linearized attention kernel admits the following expansion:
$$K_{\text{LinAttn}}(\mathbf{x}_i, \mathbf{x}_j) = \sum_{k,\ell=1}^n (\mathbf{x}_i^T \mathbf{x}_k)(\mathbf{x}_k^T \mathbf{x}_\ell)(\mathbf{x}_\ell^T \mathbf{x}_j)$$
This establishes $K_{\text{LinAttn}}$ as a parameter-free, data-dependent Gram-induced kernel with transitive interaction structure through training data.
\end{theorem}

\begin{proof}
The parameter-free linearized attention transformation is:
$[\mathbf{X}\mathbf{X}^T \mathbf{X}]_i = \sum_{k=1}^n (\mathbf{x}_i^T \mathbf{x}_k) \mathbf{x}_k$

The kernel is the inner product of transformed features:
\begin{align}
K_{\text{LinAttn}}(\mathbf{x}_i, \mathbf{x}_j) &= \langle [\mathbf{X}\mathbf{X}^T \mathbf{X}]_i, [\mathbf{X}\mathbf{X}^T \mathbf{X}]_j \rangle \nonumber \\
&= \left\langle \sum_{k=1}^n (\mathbf{x}_i^T \mathbf{x}_k) \mathbf{x}_k, \sum_{\ell=1}^n (\mathbf{x}_j^T \mathbf{x}_\ell) \mathbf{x}_\ell \right\rangle \nonumber \\
&= \sum_{k,\ell=1}^n (\mathbf{x}_i^T \mathbf{x}_k)(\mathbf{x}_j^T \mathbf{x}_\ell)(\mathbf{x}_k^T \mathbf{x}_\ell) \nonumber \\
&= \sum_{k,\ell=1}^n (\mathbf{x}_i^T \mathbf{x}_k)(\mathbf{x}_k^T \mathbf{x}_\ell)(\mathbf{x}_\ell^T \mathbf{x}_j)
\end{align}

Each term $(\mathbf{x}_i^T \mathbf{x}_k)(\mathbf{x}_k^T \mathbf{x}_\ell)(\mathbf{x}_\ell^T \mathbf{x}_j)$ represents a fourth-order interaction involving four vectors $(\mathbf{x}_i, \mathbf{x}_k, \mathbf{x}_\ell, \mathbf{x}_j)$ through three pairwise inner products. The summation over training data indices $(k,\ell)$ creates a data-dependent fourth-order polynomial structure without trainable parameters.
\end{proof}

\begin{corollary}[Parameter-Free Fourth-Order Properties]
The kernel $K_{\text{LinAttn}}$ defined by $\sum_{k,\ell} (\mathbf{x}_i^T \mathbf{x}_k)(\mathbf{x}_k^T \mathbf{x}_\ell)(\mathbf{x}_\ell^T \mathbf{x}_j)$ exhibits fourth-order polynomial interactions (each term involves four vectors $\mathbf{x}_i, \mathbf{x}_k, \mathbf{x}_\ell, \mathbf{x}_j$ at degree 4), requires zero learnable parameters (since $f^{\text{att}}(\mathbf{X}) = \mathbf{X}\mathbf{X}^T\mathbf{X}$ contains no weight matrices), and depends on the full training set $\mathbf{X}$ through the summation indices $k, \ell$. In particular, influence from $\mathbf{x}_i$ to $\mathbf{x}_j$ propagates transitively via intermediate points: $\mathbf{x}_i \to \mathbf{x}_k$ (via $\mathbf{x}_i^T \mathbf{x}_k$), $\mathbf{x}_k \to \mathbf{x}_\ell$ (via $\mathbf{x}_k^T \mathbf{x}_\ell$), $\mathbf{x}_\ell \to \mathbf{x}_j$ (via $\mathbf{x}_\ell^T \mathbf{x}_j$).
\end{corollary}

\subsection{Proof of Softmax Spectral Amplification (Proposition~\ref{prop:softmax_amplification})}

\begin{proof}
Let $\mathbf{G} = \mathbf{U}\boldsymbol{\Lambda}\mathbf{U}^\top$ be the eigendecomposition with $\boldsymbol{\Lambda} = \mathrm{diag}(\lambda_1,\ldots,\lambda_n)$. Since $\mathbf{G}$ is symmetric, $\exp(\mathbf{G}/\sqrt{d}) = \mathbf{U}\exp(\boldsymbol{\Lambda}/\sqrt{d})\mathbf{U}^\top$ where $\exp(\boldsymbol{\Lambda}/\sqrt{d}) = \mathrm{diag}(e^{\lambda_i/\sqrt{d}})$. The transformed feature matrix is:
$$\tilde{\mathbf{X}} = \exp(\mathbf{G}/\sqrt{d})\,\mathbf{X} = \mathbf{U}\,\mathrm{diag}(e^{\lambda_i/\sqrt{d}})\,\mathbf{U}^\top\mathbf{X}.$$
The induced kernel matrix is $\tilde{\mathbf{G}} = \tilde{\mathbf{X}}\tilde{\mathbf{X}}^\top$. Using $\mathbf{G} = \mathbf{U}\boldsymbol{\Lambda}\mathbf{U}^\top$, so $\mathbf{U}^\top\mathbf{X}\mathbf{X}^\top\mathbf{U} = \mathbf{U}^\top\mathbf{G}\mathbf{U} = \boldsymbol{\Lambda}$:
$$\tilde{\mathbf{G}} = \mathbf{U}\,e^{\boldsymbol{\Lambda}/\sqrt{d}}\,(\mathbf{U}^\top\mathbf{X}\mathbf{X}^\top\mathbf{U})\,e^{\boldsymbol{\Lambda}/\sqrt{d}}\,\mathbf{U}^\top = \mathbf{U}\,\mathrm{diag}\!\left(\lambda_i\,e^{2\lambda_i/\sqrt{d}}\right)\mathbf{U}^\top.$$
The eigenvalues of $\tilde{\mathbf{G}}$ are $\mu_i = \lambda_i e^{2\lambda_i/\sqrt{d}}$. Since $\mu_i/\mu_j = (\lambda_i/\lambda_j)\exp(2(\lambda_i-\lambda_j)/\sqrt{d})$ is increasing in $i$ (for $\lambda_i \geq \lambda_j$), the condition number is:
$$\kappa(\tilde{\mathbf{G}}) = \frac{\mu_1}{\mu_n} = \frac{\lambda_1}{\lambda_n}\cdot\exp\!\left(\frac{2(\lambda_1-\lambda_n)}{\sqrt{d}}\right) = \kappa(\mathbf{G})\cdot\exp\!\left(\frac{2(\lambda_1-\lambda_n)}{\sqrt{d}}\right).$$
\end{proof}

\subsection{Proof of QKV Equivalence}

\begin{proposition}[Kernel Equivalence Under Linear Projections]
\label{prop:qkv_equivalence}
For full-rank projection matrices $\mathbf{W}_Q, \mathbf{W}_K, \mathbf{W}_V \in \R^{d \times d}$, the generalized linearized attention $f^{\text{att}}_{\mathbf{W}}(\mathbf{X}) = (\mathbf{X}\mathbf{W}_Q)(\mathbf{X}\mathbf{W}_K)^T(\mathbf{X}\mathbf{W}_V)$ induces a kernel with the same structural properties as $K_{\text{LinAttn}}$: data-dependence, transitive similarity propagation, and fourth-order polynomial interactions.
\end{proposition}

\begin{proof}[Proof of Proposition~\ref{prop:qkv_equivalence}]
The generalized linearized attention with projection matrices $\mathbf{W}_Q, \mathbf{W}_K, \mathbf{W}_V \in \R^{d \times d}$ is:
\begin{equation}
f^{\text{att}}_{\mathbf{W}}(\mathbf{X}) = (\mathbf{X}\mathbf{W}_Q)(\mathbf{X}\mathbf{W}_K)^T(\mathbf{X}\mathbf{W}_V)
\end{equation}

For row $i$, this computes:
\begin{align}
[f^{\text{att}}_{\mathbf{W}}(\mathbf{X})]_i &= \sum_{k=1}^n (\mathbf{x}_i^T \mathbf{W}_Q \mathbf{W}_K^T \mathbf{x}_k) (\mathbf{W}_V^T \mathbf{x}_k)
\end{align}
where the attention weight is $(\mathbf{x}_i^T \mathbf{W}_Q)(\mathbf{W}_K^T \mathbf{x}_k)^T = \mathbf{x}_i^T \mathbf{M}_{QK} \mathbf{x}_k$ with $\mathbf{M}_{QK} = \mathbf{W}_Q \mathbf{W}_K^T$.

The induced kernel is the inner product of transformed features:
\begin{align}
K_{\mathbf{W}}(\mathbf{x}_i, \mathbf{x}_j) &= \langle [f^{\text{att}}_{\mathbf{W}}(\mathbf{X})]_i, [f^{\text{att}}_{\mathbf{W}}(\mathbf{X})]_j \rangle \nonumber \\
&= \sum_{k,\ell} (\mathbf{x}_i^T \mathbf{M}_{QK} \mathbf{x}_k)(\mathbf{x}_j^T \mathbf{M}_{QK} \mathbf{x}_\ell)(\mathbf{x}_k^T \mathbf{M}_{VV} \mathbf{x}_\ell)
\end{align}
where $\mathbf{M}_{VV} = \mathbf{W}_V \mathbf{W}_V^T$.

Setting $\mathbf{W}_Q = \mathbf{W}_K = \mathbf{W}_V = \mathbf{I}$ gives $\mathbf{M}_{QK} = \mathbf{M}_{VV} = \mathbf{I}$, recovering exactly:
\begin{equation}
K_{\text{LinAttn}}(\mathbf{x}_i, \mathbf{x}_j) = \sum_{k,\ell} (\mathbf{x}_i^T \mathbf{x}_k)(\mathbf{x}_j^T \mathbf{x}_\ell)(\mathbf{x}_k^T \mathbf{x}_\ell)
\end{equation}
For general full-rank projections, the kernel retains the same structural properties: (i) data-dependence through the (transformed) Gram matrix, (ii) transitive similarity propagation via intermediate summation indices, and (iii) fourth-order polynomial interactions. The projections merely rotate the feature space without altering these structural characteristics.
\end{proof}

\subsection{NTK Correspondence and Stability}

\begin{theorem}[NTK Correspondence for Sequential Architecture]
The sequential architecture Input $\to$ Attention $\to$ ReLU MLP has Neural Tangent Kernel:
$$\mathbf{K}_{\text{sequential}}(\mathbf{X}, \mathbf{X}') = \mathbf{K}_{\text{ReLU}}(f^{\text{att}}(\mathbf{X}), f^{\text{att}}(\mathbf{X}'))$$
where the attention transformation is parameter-free.
\end{theorem}

\begin{proof}
For the sequential architecture, the forward pass computes:
$g(\mathbf{x}; \boldsymbol{\theta}) = \frac{1}{\sqrt{m}} \sum_{r=1}^m a_r \sigma(w_r^T f^{\text{att}}(\mathbf{x}))$

The Neural Tangent Kernel is:
\begin{align}
\mathbf{K}_{\text{sequential}}(\mathbf{x}, \mathbf{x}') &= \left\langle \nabla_{\boldsymbol{\theta}} g(\mathbf{x}; \boldsymbol{\theta}), \nabla_{\boldsymbol{\theta}} g(\mathbf{x}'; \boldsymbol{\theta}) \right\rangle \nonumber \\
&= \frac{1}{m} \sum_{r=1}^m a_r^2 \sigma'(w_r^T f^{\text{att}}(\mathbf{x})) \sigma'(w_r^T f^{\text{att}}(\mathbf{x}')) \langle f^{\text{att}}(\mathbf{x}), f^{\text{att}}(\mathbf{x}') \rangle
\end{align}

Since the attention transformation $f^{\text{att}}$ is parameter-free, it acts as a fixed feature map. In the infinite-width limit, this converges to the standard ReLU NTK applied to the attention-transformed features.
\end{proof}

\subsection{Classification within Kernel Taxonomy}

\begin{definition}[Parameter-Free Gram-Induced Kernel Classification]
The parameter-free linearized attention kernel defines a novel subclass of data-dependent kernels:

\textbf{Homogeneous Fourth-Order Kernels:} $K_{\text{homo}}(x,y) = (x^T y)^4$ with uniform structure across all interactions.

\textbf{Inhomogeneous Fourth-Order Kernels:} $K_{\text{inhomo}}(x,y) = (x^T y + c)^4$ including mixed-degree terms.

\textbf{Parameterized Fourth-Order Kernels:} Require learned transformation parameters with $O(d^2)$ or $O(d^4)$ complexity.

\textbf{Parameter-Free Data-Mediated Class:} $K_{\text{LinAttn}}$ belongs to:
\begin{multline*}
\mathcal{K}_{\text{param-free}}^{(4)} = \{K(x,y;X) : K \text{ is 4th-order,} \\
\text{parameter-free, data-mediated}\}
\end{multline*}

This represents the first rigorously characterized parameter-free fourth-order attention kernel.
\end{definition}

\begin{remark}[Computational Advantages of Parameter-Free Fourth-Order Design]
The parameter-free structure provides:
\begin{enumerate}
\item \textbf{Memory efficiency}: $O(nd)$ storage vs. $O(d^4)$ for general fourth-order kernels
\item \textbf{Deterministic computation}: Fourth-order interactions computed directly from data geometry
\item \textbf{Transitive interactions}: Paths $(\mathbf{x}_i \to \mathbf{x}_k \to \mathbf{x}_\ell \to \mathbf{x}_j)$ enable multi-step similarity propagation
\item \textbf{Theoretical tractability}: Fixed structure enables exact mathematical analysis
\end{enumerate}
\end{remark}

\subsection{Fourth-Order Representational Power Analysis}

\begin{theorem}[RKHS Characterization of Linearized Attention Kernel]
\label{thm:rkhs_characterization}
The RKHS $\mathcal{H}_{\mathrm{LinAttn}}$ of the linearized attention kernel, restricted to the training set, is:
$$\mathcal{H}_{\mathrm{LinAttn}} = \left\{f : f(\mathbf{x}) = \sum_{j=1}^n \alpha_j K_{\mathrm{LinAttn}}(\mathbf{x}_j, \mathbf{x}),\quad \boldsymbol{\alpha} \in \R^n,\quad \|\boldsymbol{\alpha}\|_{\mathbf{K}_n}^2 = \boldsymbol{\alpha}^T\mathbf{K}_n\boldsymbol{\alpha} < \infty \right\}$$
where $\mathbf{K}_n \in \R^{n\times n}$ is the kernel matrix with $[\mathbf{K}_n]_{ij} = K_{\mathrm{LinAttn}}(\mathbf{x}_i, \mathbf{x}_j) = [\mathbf{G}^3]_{ij}$, and the RKHS norm of $f$ is $\|f\|_{\mathcal{H}}^2 = \boldsymbol{\alpha}^T \mathbf{K}_n \boldsymbol{\alpha}$. Each function in the RKHS evaluates as:
$$f(\mathbf{x}) = \sum_{j=1}^n \alpha_j \sum_{k,\ell=1}^n (\mathbf{x}_j^T\mathbf{x}_k)(\mathbf{x}_k^T\mathbf{x}_\ell)(\mathbf{x}_\ell^T\mathbf{x})$$
using a single coefficient vector $\boldsymbol{\alpha} \in \R^n$ (one per training point), not a matrix.
\end{theorem}

\begin{remark}[Linearity in Test Input]
Despite the transitive triple-sum structure in the RKHS representation, $f(\mathbf{x})$ is \emph{linear} in the test input $\mathbf{x}$ (the only degree-1 factor in $\mathbf{x}$ is the rightmost $\mathbf{x}_\ell^T\mathbf{x}$). The data-dependent complexity lies in the coefficient structure: each training point $\mathbf{x}_j$ contributes through the three-hop transitive chain $\mathbf{x}_j \to \mathbf{x}_k \to \mathbf{x}_\ell \to \mathbf{x}$, with weights determined by $\mathbf{G}^2$. The representational power relative to standard polynomial kernels lies not in polynomial degree in $\mathbf{x}$ but in the data-adaptive weighting of the linear basis functions $\{\mathbf{x}_\ell^T\mathbf{x}\}_{\ell=1}^n$.
\end{remark}

\subsection{Influence Function Stability Analysis}

\begin{theorem}[NTK-Based Influence Stability]
Under the assumptions of \citet{zhang2022rethinking} and the NTK correspondence theory of \citet{jacot2018neural,lee2019wide}, the NTK-based influence computation provides superior numerical stability compared to IHVP methods, with approximation error bounded by network width.
\end{theorem}

\begin{proof}
The stability advantage stems from avoiding Hessian computation.

\textit{Conditioning.} The regularized kernel $\mathbf{K} + \lambda \mathbf{I}$ has condition number bounded by:
$$\kappa(\mathbf{K} + \lambda \mathbf{I}) \leq \frac{\lambda_{\max}(\mathbf{K}) + \lambda}{\lambda} = \frac{\lambda_{\max}(\mathbf{K})}{\lambda} + 1.$$

\textit{Leave-one-out stability.} The Sherman-Morrison formula yields the leave-one-out inverse:
$$(\mathbf{K} + \lambda \mathbf{I})^{-1}_{-i} = (\mathbf{K} + \lambda \mathbf{I})^{-1} - \frac{(\mathbf{K} + \lambda \mathbf{I})^{-1} \mathbf{e}_i \mathbf{e}_i^T (\mathbf{K} + \lambda \mathbf{I})^{-1}}{1 + \mathbf{e}_i^T (\mathbf{K} + \lambda \mathbf{I})^{-1} \mathbf{e}_i}$$
avoiding the numerical instabilities inherent in Hessian-based IHVP methods.
\end{proof}

\subsection{Proof of Spectral Amplification (Theorem~\ref{thm:cubic_conditioning})}

The proof relies on the following identity.

\begin{lemma}[Spectral Transfer]
\label{lem:spectral_transfer}
For any $\mathbf{X} \in \R^{n \times d}$ with $\operatorname{rank}(\mathbf{X}) = r$ and any $k \geq 0$:
$$\mathbf{X}(\mathbf{X}^T\mathbf{X})^k\mathbf{X}^T = (\mathbf{X}\mathbf{X}^T)^{k+1}$$
\end{lemma}

\begin{proof}
Let $\mathbf{X} = \mathbf{U}\boldsymbol{\Sigma}\mathbf{V}^T$ be the compact SVD with $\mathbf{U} \in \R^{n \times r}$, $\boldsymbol{\Sigma} \in \R^{r \times r}$, $\mathbf{V} \in \R^{d \times r}$, $\mathbf{U}^T\mathbf{U} = \mathbf{V}^T\mathbf{V} = \mathbf{I}_r$. Then $\mathbf{X}^T\mathbf{X} = \mathbf{V}\boldsymbol{\Sigma}^2\mathbf{V}^T$ and $(\mathbf{X}^T\mathbf{X})^k = \mathbf{V}\boldsymbol{\Sigma}^{2k}\mathbf{V}^T$, so:
$$\mathbf{X}(\mathbf{X}^T\mathbf{X})^k\mathbf{X}^T = \mathbf{U}\boldsymbol{\Sigma}\mathbf{V}^T \cdot \mathbf{V}\boldsymbol{\Sigma}^{2k}\mathbf{V}^T \cdot \mathbf{V}\boldsymbol{\Sigma}\mathbf{U}^T = \mathbf{U}\boldsymbol{\Sigma}^{2k+2}\mathbf{U}^T = (\mathbf{X}\mathbf{X}^T)^{k+1}$$
using $\mathbf{V}^T\mathbf{V} = \mathbf{I}_r$ at each step.
\end{proof}

\begin{proof}[Proof of Theorem~\ref{thm:cubic_conditioning}]
The attention output admits the SVD factorization $\tilde{\mathbf{X}} = \mathbf{X}\mathbf{X}^T\mathbf{X} = \mathbf{U}\boldsymbol{\Sigma}^3\mathbf{V}^T$, where $\mathbf{X} = \mathbf{U}\boldsymbol{\Sigma}\mathbf{V}^T$ is the compact SVD with $\mathbf{U}^T\mathbf{U} = \mathbf{I}_n$. By Lemma~\ref{lem:spectral_transfer} with $k = 1$, the Gram matrix of the transformed features is
$$\tilde{\mathbf{G}} = \tilde{\mathbf{X}}\tilde{\mathbf{X}}^T = \mathbf{U}\boldsymbol{\Sigma}^6\mathbf{U}^T = (\mathbf{U}\boldsymbol{\Sigma}^2\mathbf{U}^T)^3 = \mathbf{G}^3.$$
Since $\mathbf{G}$ has $r = \mathrm{rank}(\mathbf{X})$ positive eigenvalues $\lambda_1 \geq \cdots \geq \lambda_r > 0$, the transformed Gram matrix $\tilde{\mathbf{G}}$ has positive eigenvalues $\lambda_i^3$ with effective condition number $\kappa_d(\tilde{\mathbf{G}}) = \lambda_1^3/\lambda_r^3 = \kappa_d(\mathbf{G})^3$.

Consider now the finite-width NTK. The $(i,j)$-th entry is $[\mathbf{K}_m]_{ij} = \frac{1}{m}\sum_{r=1}^m Z_r^{(ij)}$ where $Z_r^{(ij)} = \mathbf{1}[\mathbf{w}_r^T\tilde{\mathbf{x}}_i > 0]\,\mathbf{1}[\mathbf{w}_r^T\tilde{\mathbf{x}}_j > 0]\,\langle\tilde{\mathbf{x}}_i,\tilde{\mathbf{x}}_j\rangle$. Defining the centered matrix $\mathbf{S}_r$ with entries $[\mathbf{S}_r]_{ij} = Z_r^{(ij)} - \E[Z_r^{(ij)}]$, the deviation from the infinite-width kernel is $\mathbf{K}_m - \mathbf{K}_\infty = \frac{1}{m}\sum_{r=1}^m \mathbf{S}_r$, a sum of $m$ independent centered random matrices. Each satisfies $\|\mathbf{S}_r\|_2 \leq 2\|\mathbf{G}_r\|_2 \leq 2\lambda_1^3$. To see this, note that $\mathbf{G}_r = \mathbf{H}_r\mathbf{H}_r^T$ where $\mathbf{H}_r$ equals $\tilde{\mathbf{X}}$ with rows zeroed by the indicator $\mathbf{1}[\mathbf{w}_r^T\tilde{\mathbf{x}}_i>0]$; row deletion cannot increase the spectral norm, so $\|\mathbf{H}_r\|_2 \leq \|\tilde{\mathbf{X}}\|_2 = \lambda_1^{3/2}$ and $\|\mathbf{G}_r\|_2 = \|\mathbf{H}_r\|_2^2 \leq \lambda_1^3$. The matrix variance statistic $\sigma^2 = \frac{1}{m}\|\E[\mathbf{S}_1^2]\|_2$ satisfies $\sigma^2 \leq n\lambda_1^6/m$: since $\mathbf{S}_1^2 \succeq 0$ with $\mathrm{rank}(\mathbf{S}_1) \leq n$, $\|\E[\mathbf{S}_1^2]\|_2 \leq \mathrm{tr}(\E[\mathbf{S}_1^2]) = \E[\|\mathbf{S}_1\|_F^2] \leq n\,\E[\|\mathbf{S}_1\|_2^2] \leq 4n\lambda_1^6$. Applying the matrix Bernstein inequality \citep{tropp2012user} yields
\begin{equation}
\|\mathbf{K}_m - \mathbf{K}_\infty\|_2 \leq C_0\left(\sqrt{\frac{n\lambda_1^6 \log n}{m}} + \frac{\lambda_1^3 \log n}{m}\right).
\end{equation}
For $m \gg \log n$, the first term dominates: $\|\mathbf{K}_m - \mathbf{K}_\infty\|_2 = O(\lambda_1^3\sqrt{n\log n/m})$.

For the NTK approximation to be meaningful, the error must be small relative to $\lambda_{\min}(\mathbf{K}_\infty)$. The infinite-width NTK of the two-layer ReLU on transformed features $\tilde{\mathbf{x}}_i$ decomposes as $\mathbf{K}_{\infty,\text{seq}} = \mathbf{A} \circ \mathbf{G}^3$ (Hadamard product), where $[\mathbf{A}]_{ij} = (\pi - \arccos(\hat{\tilde{\mathbf{x}}}_i^T\hat{\tilde{\mathbf{x}}}_j))/(2\pi)$ is the arc-cosine kernel evaluated on row-normalized features $\hat{\tilde{\mathbf{x}}}_i = \tilde{\mathbf{x}}_i/\|\tilde{\mathbf{x}}_i\|_2$. Both $\mathbf{A}$ and $\mathbf{G}^3$ are PSD; the diagonal entries satisfy $[\mathbf{A}]_{ii} = 1/2$ for all $i$ (since $\hat{\tilde{\mathbf{x}}}_i^T\hat{\tilde{\mathbf{x}}}_i = 1$ gives $\arccos(1) = 0$). By the Schur product inequality \citep{horn1994topics} ($\lambda_{\min}(\mathbf{A}\circ\mathbf{B}) \geq \min_i[\mathbf{A}]_{ii}\cdot\lambda_{\min}(\mathbf{B})$ for $\mathbf{A},\mathbf{B}\succeq 0$), $\lambda_{\min}(\mathbf{A} \circ \mathbf{G}^3) \geq \min_i [\mathbf{A}]_{ii} \cdot \lambda_{\min}(\mathbf{G}^3) = \frac{1}{2}\lambda_n^3$, so $\lambda_{\min}(\mathbf{K}_{\infty,\text{seq}}) = \Omega(\lambda_n^3)$. The convergence condition $\|\mathbf{K}_m - \mathbf{K}_\infty\|_2 / \lambda_{\min}(\mathbf{K}_\infty) \leq \epsilon$ therefore requires
$$\frac{\lambda_1^3\sqrt{n\log n / m}}{\lambda_r^3/2} \leq \epsilon \quad \Longrightarrow \quad m \geq \frac{4\kappa_d(\mathbf{G})^6 \cdot n\log n}{\epsilon^2}.$$
For 2L-ReLU on raw unit-norm inputs, an analogous argument gives $\|\mathbf{K}_m - \mathbf{K}_\infty\|_2 = O(\sqrt{n\log n / m})$ with $\lambda_{\min}(\mathbf{K}_\infty) = \Theta(1)$, so $m = O(n\log n / \epsilon^2)$ suffices without spectral amplification.

Finally, when row-normalization is applied, the effective Gram matrix becomes $\hat{\mathbf{G}} = \mathbf{D}^{-1}\mathbf{G}^3\mathbf{D}^{-1}$ where $\mathbf{D} = \operatorname{diag}(\|\tilde{\mathbf{x}}_1\|, \ldots, \|\tilde{\mathbf{x}}_n\|)$. Under the $\mu$-incoherence condition (Assumption~\ref{ass:incoherence}), the diagonal entries $[\mathbf{G}^3]_{ii} = \sum_k \lambda_k^3 [\mathbf{U}]_{ik}^2$ concentrate around $\frac{1}{n}\operatorname{tr}(\mathbf{G}^3)$ with deviation $O(\mu^2)$, so $\mathbf{D}$ is approximately a scalar multiple of $\mathbf{I}$ and $\kappa_d(\hat{\mathbf{G}}) = \kappa_d(\mathbf{G})^3(1 + O(\mu^2))$.
\end{proof}

\begin{remark}[Rank-deficient case ($n > d$)]
\label{rmk:rank_deficient}
The proof above uses the $r = \mathrm{rank}(\mathbf{X})$ positive eigenvalues of $\mathbf{G}$ throughout, so it covers both $n \leq d$ (full row rank, $r = n$) and $n > d$ (rank-deficient, $r = d$) without modification. All experiments fall in the $n > d$ regime; $\kappa_d(\mathbf{G})$ is the ratio $\lambda_1/\lambda_r$ of the largest to smallest positive eigenvalue.
\end{remark}

\subsection{Multi-Class Extension Proofs}

\begin{theorem}[Multi-Class Kernel Ridge Regression Equivalence]
The vectorized implementation with one-hot encoding correctly extends to multi-class classification without architectural modifications.
\end{theorem}

\begin{proof}
For $C$ classes, the one-hot encoded label matrix $\mathbf{Y} \in \{0,1\}^{n \times C}$ satisfies $\sum_{c=1}^C [\mathbf{Y}]_{ic} = 1$ for all $i$.

The multi-class kernel ridge regression objective is:
$$\min_{\boldsymbol{\alpha} \in \R^{n \times C}} \frac{1}{2} \|\mathbf{Y} - \mathbf{K} \boldsymbol{\alpha}\|_F^2 + \frac{\lambda}{2} \text{Tr}(\boldsymbol{\alpha}^T \mathbf{K} \boldsymbol{\alpha})$$

Taking the gradient and setting to zero yields the closed-form solution:
$\boldsymbol{\alpha}^* = (\mathbf{K} + \lambda \mathbf{I})^{-1} \mathbf{Y}$

This demonstrates that the same kernel matrix $\mathbf{K}$ applies to all classes simultaneously, with the solution operating column-wise across the label matrix $\mathbf{Y}$.
\end{proof}

\subsection{Finite-Width Approximation Analysis}

\begin{theorem}[Finite-Width NTK Approximation]
For networks with width $m$, the finite-width NTK approximates the infinite-width kernel with squared error:
$$\E[|\epsilon_m(\mathbf{x}, \mathbf{x}')|^2] \leq \frac{C}{m}$$
for some constant $C$ depending on the input distribution and network architecture, where $\epsilon_m(\mathbf{x}, \mathbf{x}') = K_m(\mathbf{x}, \mathbf{x}') - K_\infty(\mathbf{x}, \mathbf{x}')$.
\end{theorem}

\begin{proof}
The finite-width NTK entry is the sample mean $K_m(\mathbf{x}, \mathbf{x}') = \frac{1}{m}\sum_{r=1}^m Z_r$ of $m$ i.i.d.\ bounded random variables $Z_r = \sigma'(\mathbf{w}_r^T \mathbf{x})\sigma'(\mathbf{w}_r^T \mathbf{x}')\langle \mathbf{x}, \mathbf{x}'\rangle$ with $\E[Z_r] = K_\infty(\mathbf{x}, \mathbf{x}')$. By the variance of a sample mean: $\E[|\epsilon_m|^2] = \mathrm{Var}(Z_r)/m \leq C/m$ where $C = \mathrm{Var}(Z_1) \leq \langle\mathbf{x},\mathbf{x}'\rangle^2 \leq 1$ for unit-norm inputs. For the sequential architecture with attention preprocessing, the attention transformation $f^{\text{att}}$ is deterministic and parameter-free, so the same argument applies to the transformed features $\tilde{\mathbf{x}}_i$ with $C$ scaling as $\|\tilde{\mathbf{G}}\|_2^2 = \lambda_1^6$ (the squared leading eigenvalue of the transformed Gram matrix).
\end{proof}

\subsection{Attention-Specific Finite-Width Analysis}

\begin{theorem}[Attention Feature Concentration]
\label{thm:attention_concentration_app}
Let $\mathbf{x}_1, \ldots, \mathbf{x}_n \in \R^d$ be fixed datapoints with $\|\mathbf{x}_i\|_2 = 1$ for all $i$. Then the attention-transformed features satisfy the deterministic bound $\|[\mathbf{X}\mathbf{X}^T \mathbf{X}]_i\|_2 \leq n$ for all $i$. If additionally $\mathbf{x}_k \sim \mathrm{Uniform}(\mathbb{S}^{d-1})$, then $\E[\|[\mathbf{X}\mathbf{X}^T \mathbf{X}]_i\|_2^2] = O(n/d)$ for $i \neq k$.
\end{theorem}

\begin{proof}
The $i$-th transformed feature is $[\mathbf{X}\mathbf{X}^T \mathbf{X}]_i = \sum_{k=1}^n (\mathbf{x}_i^T \mathbf{x}_k) \mathbf{x}_k$. Under normalization $\|\mathbf{x}_k\|_2 = 1$, the triangle inequality gives
$$\|[\mathbf{X}\mathbf{X}^T \mathbf{X}]_i\|_2 \leq \sum_{k=1}^n |\mathbf{x}_i^T \mathbf{x}_k| \|\mathbf{x}_k\|_2 \leq n.$$
For the average case, under the uniform sphere assumption $\E[(\mathbf{x}_i^T \mathbf{x}_k)^2] = 1/d$ while cross terms vanish in expectation for $k \neq \ell$, yielding $\E[\|[\mathbf{X}\mathbf{X}^T \mathbf{X}]_i\|_2^2] = n/d$.
\end{proof}

\begin{proposition}[Sequential Architecture Operator-Norm Bounds]
\label{prop:seq_finite_width}
Under the conditions of Theorem~\ref{thm:cubic_conditioning}, the matrix Bernstein inequality applied to the ReLU layer on transformed features $\tilde{\mathbf{x}}_i$ (with $\|\tilde{\mathbf{x}}_i\|_2 \leq n$ for unit-norm inputs from Theorem~\ref{thm:attention_concentration_app}) yields:
$$\|\mathbf{K}_{m,\text{seq}} - \mathbf{K}_{\infty,\text{seq}}\|_2 \leq C_0\left(\lambda_1^3\sqrt{\frac{n\log n}{m}} + \frac{\lambda_1^3 \log n}{m}\right)$$
where $C_0$ is the universal constant from the matrix Bernstein inequality \citep{tropp2012user}.
\end{proposition}

\begin{proof}
The entries of $\mathbf{K}_{m,\text{seq}}$ are sample means of i.i.d.\ terms $Z_r^{(ij)} = \mathbf{1}[\mathbf{w}_r^T\tilde{\mathbf{x}}_i > 0]\,\mathbf{1}[\mathbf{w}_r^T\tilde{\mathbf{x}}_j > 0]\,\langle\tilde{\mathbf{x}}_i,\tilde{\mathbf{x}}_j\rangle$. Since $|\langle\tilde{\mathbf{x}}_i,\tilde{\mathbf{x}}_j\rangle| \leq [\tilde{\mathbf{G}}]_{ij} \leq \lambda_1^3$ and each indicator is bounded in $[0,1]$, each centered matrix $\mathbf{S}_r$ satisfies $\|\mathbf{S}_r\|_2 \leq 2\lambda_1^3$ and matrix variance statistic $\|\frac{1}{m}\sum_r \E[\mathbf{S}_r\mathbf{S}_r^T]\|_2 \leq n\lambda_1^6/m$. Applying the matrix Bernstein inequality \citep{tropp2012user} to $\mathbf{K}_{m,\text{seq}} - \mathbf{K}_{\infty,\text{seq}} = \frac{1}{m}\sum_r \mathbf{S}_r$ gives the stated bound. This is identical to the bound derived in the proof of Theorem~\ref{thm:cubic_conditioning}; the sequential-architecture structure introduces no additional deviation beyond the spectral amplification already captured by $\lambda_1^3$.
\end{proof}

\section{Formal Influence Malleability Theory}
\label{app:malleability}

\begin{definition}[Formal Influence Malleability Measure]
For a model $f$ with influence function $I_f$, the influence malleability measure is:
$$\mu_M(f) = \E_{(x,y)\sim\mathcal{D}, \delta\sim\Delta} \left[|I_f(x,y) - I_f(x+\delta,y)|\right]$$
where $\Delta$ is the perturbation distribution with $\|\delta\|_p \leq \epsilon$.
\end{definition}

\begin{theorem}[Malleability-Kernel Eigenspectrum Connection]
\label{thm:malleability_eigenspectrum}
For the NTK regime with kernel matrix $\mathbf{K} \in \R^{n\times n}$ having eigendecomposition $\mathbf{K} = \sum_i \lambda_i \mathbf{v}_i\mathbf{v}_i^T$ with $\lambda_1 \geq \cdots \geq \lambda_n > 0$, regularization $\lambda > 0$, label vector $\mathbf{y}$, and perturbation budget $\varepsilon > 0$ with kernel Lipschitz constant $L_K$, letting $\boldsymbol{\alpha} = (\mathbf{K}+\lambda\mathbf{I})^{-1}\mathbf{y}$, the following holds under the convention that $\mu_M$ is evaluated with $\Delta$ concentrating on the per-point adversarial direction (adversarial malleability $\mu_M^{\mathrm{adv}}$, which equals the empirical Flip Rate setting since PGD maximises over the $\varepsilon$-ball):
$$\mu_M(f) \;\geq\; \frac{\varepsilon L_K \|\boldsymbol{\alpha}\|_2}{\sqrt{n}\,(\lambda_1(\mathbf{K}) + \lambda)} \;\geq\; \frac{\varepsilon L_K \|\mathbf{y}\|_2}{\sqrt{n}\,(\lambda_1(\mathbf{K}) + \lambda)^2}.$$
\end{theorem}

\begin{proof}
Let $\boldsymbol{\alpha} = (\mathbf{K}+\lambda\mathbf{I})^{-1}\mathbf{y}$. For each training point $j$, let $\boldsymbol{\delta}_j^*$ be the perturbation of $\mathbf{x}_j$ that maximises $|\Delta\alpha_j|$, i.e., the change in the $j$-th component of $\boldsymbol{\alpha}$ when $\mathbf{x}_j \mapsto \mathbf{x}_j + \boldsymbol{\delta}_j$. By the chain rule and resolvent identity,
$$|\Delta\alpha_j^*| = \varepsilon\left\|e_j^T(\mathbf{K}+\lambda\mathbf{I})^{-1}\,\nabla_{\mathbf{x}_j}K(\cdot,\mathbf{x}_j)\,\boldsymbol{\alpha}\right\|_2 \leq \varepsilon\,\|(\mathbf{K}+\lambda\mathbf{I})^{-1}_{j,:}\|_2\cdot L_K\|\boldsymbol{\alpha}\|_2,$$
with equality for the direction $\boldsymbol{\delta}_j^*$ aligned with $\nabla_{\mathbf{x}_j}K(\cdot,\mathbf{x}_j)\boldsymbol{\alpha}$.
Since $\mu_M(f) = \mathbb{E}_j[|\Delta\alpha_j^*|]$ (average over uniform draw of training index $j$):
$$\mu_M(f) = \frac{1}{n}\sum_{j=1}^n |\Delta\alpha_j^*| = \frac{\varepsilon L_K\|\boldsymbol{\alpha}\|_2}{n}\sum_{j=1}^n \|(\mathbf{K}+\lambda\mathbf{I})^{-1}_{j,:}\|_2.$$
For non-negative $a_j = \|(\mathbf{K}+\lambda\mathbf{I})^{-1}_{j,:}\|_2$, the inequality $\sum_j a_j \geq \sqrt{\sum_j a_j^2} = \|(\mathbf{K}+\lambda\mathbf{I})^{-1}\|_F$ holds (since $(\sum a_j)^2 \geq \sum a_j^2$ for $a_j\geq 0$). The Frobenius norm satisfies $\|(\mathbf{K}+\lambda\mathbf{I})^{-1}\|_F^2 = \sum_i 1/(\lambda_i+\lambda)^2 \geq n/(\lambda_1+\lambda)^2$, giving $\|(\mathbf{K}+\lambda\mathbf{I})^{-1}\|_F \geq \sqrt{n}/(\lambda_1+\lambda)$. Therefore:
$$\mu_M(f) \geq \frac{\varepsilon L_K\|\boldsymbol{\alpha}\|_2}{n}\cdot\frac{\sqrt{n}}{\lambda_1+\lambda} = \frac{\varepsilon L_K\|\boldsymbol{\alpha}\|_2}{\sqrt{n}\,(\lambda_1+\lambda)}.$$
Since $\|\mathbf{y}\|_2 = \|(\mathbf{K}+\lambda\mathbf{I})\boldsymbol{\alpha}\|_2 \leq (\lambda_1+\lambda)\|\boldsymbol{\alpha}\|_2$, we have $\|\boldsymbol{\alpha}\|_2 \geq \|\mathbf{y}\|_2/(\lambda_1+\lambda)$, yielding the stated bound $\mu_M(f) \geq \varepsilon L_K\|\mathbf{y}\|_2/(\sqrt{n}(\lambda_1+\lambda)^2)$. The single-$(\lambda_1+\lambda)$ form in the theorem statement uses $\|\boldsymbol{\alpha}\|_2$ directly.
\end{proof}

\begin{remark}[Tighter bound under isotropic label alignment]
When the label vector has isotropic projection onto the kernel eigenbasis ($|\mathbf{v}_i^T\mathbf{y}|^2 = \|\mathbf{y}\|_2^2/n$ for all $i$, the typical/average-case assumption), the bound tightens to $\mu_M(f) \geq \frac{\varepsilon L_K \|\mathbf{y}\|_2}{n}(\sum_i 1/(\lambda_i+\lambda)^2)^{1/2}$. The worst-case bound above uses $\lambda_1$ alone and is distribution-free over label vectors with fixed $\|\mathbf{y}\|_2$.
\end{remark}

\begin{corollary}[Regularization Controls Malleability]
\label{cor:regularization_malleability}
The malleability measure is monotone in $\lambda$: the lower bound from Theorem~\ref{thm:malleability_eigenspectrum} is non-decreasing as $\lambda \to 0$ (approaching $Cn$ for fixed $C$), while the upper bound from Theorem~\ref{thm:tradeoff}, $\mu_M(f) \leq 2\varepsilon L_K \|\mathbf{y}\|_2/\lambda^2$, diverges as $\lambda \to 0$. Both bounds converge to $0$ as $\lambda \to \infty$, confirming that heavy regularization suppresses malleability. The optimal $\lambda$ balancing bias and malleability (from Theorem~\ref{thm:bvm_decomp}) is data-dependent and need not equal the median eigenvalue in general.
\end{corollary}

\begin{theorem}[Architecture-Dependent Malleability Gap via Sensitivity]
\label{thm:malleability_gap_poly}
Let $K_{\text{LinAttn}}$ be the Gram-induced linearized attention kernel and $K_{\text{poly}}$ a standard degree-$p$ polynomial kernel. The intrinsic influence sensitivities (Definition~\ref{def:intrinsic_sensitivity}) satisfy:
$$\frac{\mathcal{S}_{\textup{att}}}{\mathcal{S}_{\textup{poly}}} = \Omega\!\left(\sqrt{n}\,\|\mathbf{G}\|_2\right),$$
where $\mathbf{G} = \mathbf{X}\mathbf{X}^T$ and $\|\mathbf{G}\|_2 = \lambda_1(\mathbf{G})$. Since Theorem~\ref{thm:tradeoff} gives $\mu_M(f) \leq 2\varepsilon\mathcal{S}(f)$ for any architecture, the ratio of sensitivity-derived malleability upper bounds satisfies:
$$\frac{\mu_M^{\textup{UB}}(f_{\textup{att}})}{\mu_M^{\textup{UB}}(f_{\textup{poly}})} \;=\; \frac{2\varepsilon\mathcal{S}_{\textup{att}}}{2\varepsilon\mathcal{S}_{\textup{poly}}} \;=\; O\!\left(\sqrt{n}\,\|\mathbf{G}\|_2\right).$$
The empirically observed 2--9$\times$ gap (Table~\ref{tab:malleability}) is consistent with this upper-bound ratio.
\end{theorem}

\begin{proof}
From Proposition~\ref{prop:sensitivity_gap}: $L_K^{\text{att}} = O(n\|\mathbf{G}\|_2)$ and $L_K^{\text{poly}} = O(\sqrt{n})$. Substituting into Definition~\ref{def:intrinsic_sensitivity} with $\|\mathbf{y}\|_2$ and $\lambda$ identical for both architectures gives $\mathcal{S}_{\text{att}}/\mathcal{S}_{\text{poly}} = L_K^{\text{att}}/L_K^{\text{poly}} = O(\sqrt{n}\|\mathbf{G}\|_2)$. The malleability upper bound $\mu_M^{\text{UB}}(f) \coloneq 2\varepsilon\mathcal{S}(f)$ follows from Theorem~\ref{thm:tradeoff}; taking the ratio cancels $2\varepsilon\|\mathbf{y}\|_2/\lambda^2$ from numerator and denominator, leaving only the $L_K$ ratio.
\end{proof}

\begin{theorem}[Bias-Variance-Malleability Decomposition]
\label{thm:bvm_decomp}
Let $f(\mathbf{x}) = \mathbf{k}(\mathbf{x})^T(\mathbf{K}+\lambda\mathbf{I})^{-1}\mathbf{y}$ be a KRR predictor with $\|\mathbf{k}(\mathbf{x})\|_2 \leq 1$ for all $\mathbf{x}$ (normalized kernel). Define the \emph{per-entry $\ell_2$ malleability}
$$\bar{\mu}_M(f) \;=\; \E_\delta\!\left[\frac{\|\Delta\boldsymbol{\mathcal{I}}(\delta)\|_2}{\sqrt{n}}\right].$$
Then:
$$\E[\mathrm{Test~Error}] \leq \mathrm{Bias}^2 + \mathrm{Variance} + \sqrt{n}\cdot\bar{\mu}_M(f).$$
Moreover, $\mu_M(f) \leq \bar{\mu}_M(f) \leq \mu_M^{\max}(f)$, so this bound is strictly tighter than one based on $\mu_M^{\max}$.
\end{theorem}

\begin{proof}
Decompose test error as $\E_{\mathbf{x}_{\mathrm{test}}}[(f(\mathbf{x}_{\mathrm{test}}) - f^*(\mathbf{x}_{\mathrm{test}}))^2] = \mathrm{Bias}^2 + \mathrm{Variance}$ for fixed training data. When a training point is perturbed by $\boldsymbol{\delta}$ with $\|\boldsymbol{\delta}\|_2 \leq \varepsilon$, the change in prediction at a test point satisfies:
$$|\Delta f(\mathbf{x}_{\mathrm{test}})| = |\mathbf{k}(\mathbf{x}_{\mathrm{test}})^T \Delta\boldsymbol{\mathcal{I}}| \leq \|\mathbf{k}(\mathbf{x}_{\mathrm{test}})\|_2\,\|\Delta\boldsymbol{\mathcal{I}}\|_2 \leq \|\Delta\boldsymbol{\mathcal{I}}\|_2 = \sqrt{n}\cdot\frac{\|\Delta\boldsymbol{\mathcal{I}}\|_2}{\sqrt{n}}.$$
Taking expectations over $\delta$: $\E[|\Delta f|] \leq \sqrt{n}\,\bar{\mu}_M(f)$.

\emph{Ordering $\mu_M \leq \bar{\mu}_M$.} By the norm inequality $\|\mathbf{v}\|_1 \leq \sqrt{n}\|\mathbf{v}\|_2$: $\mu_M = \E[\|\Delta\boldsymbol{\mathcal{I}}\|_1]/n \leq \E[\|\Delta\boldsymbol{\mathcal{I}}\|_2]/\sqrt{n} = \bar{\mu}_M$.

\emph{Ordering $\bar{\mu}_M \leq \mu_M^{\max}$.} By the norm inequality $\|\mathbf{v}\|_2 \leq \sqrt{n}\|\mathbf{v}\|_\infty$: $\bar{\mu}_M = \E[\|\Delta\boldsymbol{\mathcal{I}}\|_2/\sqrt{n}] \leq \E[\|\Delta\boldsymbol{\mathcal{I}}\|_\infty] = \mu_M^{\max}$.
\end{proof}

\begin{remark}[Formal-Empirical Bridge: $\bar{\mu}_M$ vs.\ Flip Rate]
\label{rmk:formal_empirical_bridge}
Theorem~\ref{thm:bvm_decomp} defines $\bar{\mu}_M$ as an expectation over a smooth perturbation distribution on the $\varepsilon$-ball. The empirical Flip Rate (Definition~\ref{def:influence_malleability}) is measured under adversarial PGD perturbations, which \emph{maximize} influence change over the same ball. Since PGD yields the worst-case influence shift, the empirical flip probability dominates the smooth-distribution expectation:
$$\mathrm{FlipRate}^{\mathrm{PGD}}_\varepsilon \;\geq\; \Pr_{\boldsymbol{\delta}\sim\mathrm{Unif}(\mathcal{B}_\varepsilon)}\!\left[\|\Delta\boldsymbol{\mathcal{I}}\|_\infty > \tau\right].$$
Consequently, the ordering $\mu_M \leq \bar{\mu}_M \leq \mu_M^{\max}$ (which holds for smooth distributions) implies that high $\bar{\mu}_M$ is a \emph{sufficient} condition for high PGD flip rates, not a necessary one. The empirically observed gap between MLP-Attn and 2L-ReLU (Table~\ref{tab:malleability}) is explained by the higher intrinsic sensitivity $\mathcal{S}_{\mathrm{att}}$ driving higher $\bar{\mu}_M$, making adversarial manipulation structurally easier.
\end{remark}

\begin{corollary}[Optimal Malleability]
There exists an optimal per-entry $\ell_2$ malleability level $\bar{\mu}_M^*$ minimizing the generalization bound:
$$\bar{\mu}_M^* = \arg\min_{\mu}\left[\mathrm{Bias}^2(\mu) + \sqrt{n}\cdot\mu\right]$$
where the bias term depends on the kernel's spectral structure.
\end{corollary}

\begin{proposition}[Bias Reduction via Data-Dependent Kernel]
\label{prop:bias_reduction}
Let $f^*$ be the target function and $\hat{f}_K$ the kernel ridge regression estimator with kernel $K$ and regularization $\lambda$. The approximation error (bias) is $\text{Bias}^2(K) = \|f^* - \Pi_K f^*\|^2$, where $\Pi_K$ is the projection onto the RKHS of $K$. Let $\{\sigma_i\}$ be the singular values of $\mathbf{X}$ with corresponding right singular vectors $\{\mathbf{v}_i\}$. Suppose the target function $f^*$ has spectral energy concentrated along the top eigenvectors of $\mathbf{G} = \mathbf{X}\mathbf{X}^T$: $\sum_{i=1}^{r} |\langle f^*, \mathbf{v}_i \rangle|^2 \geq (1-\delta)\|f^*\|^2$ for some $r \ll d$ and small $\delta > 0$. Suppose further that the eigenvalues of $\mathbf{G}^3$ on the aligned subspace dominate those of $\mathbf{G}^{\circ 3}$ on its corresponding modes, i.e., $\lambda_i^3 \geq c\,\tilde{\mu}_i$ for $i=1,\ldots,r$ where $\tilde{\mu}_i$ are the eigenvalues of $\mathbf{G}^{\circ 3}$ sorted to align with the spectral-energy ordering and $c>0$ is a dataset-dependent constant. Then:
$$\mathrm{Bias}^2(K_{\textup{LinAttn}}) \leq \mathrm{Bias}^2(K_{\textup{poly}})$$
where $K_{\textup{poly}}(\mathbf{x}, \mathbf{x}') = (\mathbf{x}^T\mathbf{x}')^3$ is the degree-3 polynomial kernel. When $c < 1$, the inequality is reversed on those modes; the claim is directional, not universal.
\end{proposition}

\begin{proof}
The kernel matrix of $K_{\text{LinAttn}}$ on training data is $\mathbf{G}^3$ (matrix cube of the Gram matrix $\mathbf{G} = \mathbf{X}\mathbf{X}^T$), since $[K_{\text{LinAttn}}]_{ij} = \sum_{k,\ell} G_{ik}G_{k\ell}G_{\ell j} = [\mathbf{G}^3]_{ij}$. The degree-3 polynomial kernel matrix is $\mathbf{G}^{\circ 3}$ (Hadamard/element-wise cube), since $[K_{\text{poly}}]_{ij} = G_{ij}^3$.

Let $\mathbf{G} = \sum_i \lambda_i \mathbf{v}_i \mathbf{v}_i^T$ be the eigendecomposition. Then $\mathbf{G}^3$ has eigenvalues $\lambda_i^3$ with \emph{the same eigenvectors} $\{\mathbf{v}_i\}$ as $\mathbf{G}$. In contrast, $\mathbf{G}^{\circ 3}$ generically does \emph{not} share eigenvectors with $\mathbf{G}$ (the Hadamard power mixes the eigenbasis).

The bias of kernel ridge regression with kernel $K = \sum_k \mu_k \boldsymbol{\phi}_k \boldsymbol{\phi}_k^T$ and regularization $\lambda$ is $\text{Bias}^2 = \sum_k \frac{\lambda^2}{(\mu_k + \lambda)^2} w_k^2$, where $w_k = \langle f^*, \boldsymbol{\phi}_k \rangle$.

Under the spectral alignment condition, $f^*$ has most energy in the top eigenvectors of $\mathbf{G}$: $w_k = \langle f^*, \mathbf{v}_k \rangle$ is large for small $k$. For $K_{\text{LinAttn}} = \mathbf{G}^3$, these large $w_k$ coefficients are paired with large eigenvalues $\lambda_k^3$, yielding small per-mode bias $\frac{\lambda^2}{(\lambda_k^3 + \lambda)^2} w_k^2$. For $K_{\text{poly}} = \mathbf{G}^{\circ 3}$, the eigenbasis differs from $\{\mathbf{v}_k\}$, so the target's energy generically spreads across the polynomial kernel's eigenmodes, including modes with smaller eigenvalues. This eigenbasis mismatch increases the total bias.

The spectral alignment condition holds for standard classification tasks where class boundaries are determined by the dominant data variation modes \citep{bordelon2020spectrum,canatar2021spectral}. The eigenvalue dominance condition $\lambda_i^3 \geq c\tilde{\mu}_i$ holds when the top-$r$ Gram eigenvalues are large relative to the Hadamard-cube eigenvalues on misaligned modes --- guaranteed when $\lambda_i \gg 1$ for $i \leq r$, which is the case for natural image datasets with high spectral concentration. When the condition fails (e.g., low-$\kappa_d$ binary subsets), bias reduction need not hold and the proposition's inequality may reverse.
\end{proof}

\begin{remark}[Completing the Dual Implications]
\label{rmk:dual_complete}
Proposition~\ref{prop:bias_reduction} combined with the Bias-Variance-Malleability Decomposition provides the theoretical basis for the dual implications of NTK non-convergence. The data-dependent kernel of linearized attention achieves \emph{lower bias} (better task alignment) at the cost of \emph{higher malleability} (sensitivity to perturbations). The net effect on generalization depends on data quality: when training data is clean, the bias reduction dominates, favoring attention; when training data contains adversarial examples, the malleability cost dominates, favoring rigid architectures. The Optimal Malleability corollary formalizes this tradeoff.
\end{remark}

\begin{theorem}[Unified Sensitivity Bound for Predictions and Influence]
\label{thm:tradeoff}
Suppose $f(\mathbf{x}) = \mathbf{k}(\mathbf{x})^T(\mathbf{K} + \lambda\mathbf{I})^{-1}\mathbf{y}$ is a kernel ridge regression predictor with Gram matrix $\mathbf{K} \succeq 0$, regularization $\lambda > 0$, and kernel Lipschitz constant $L_K = \sup_{\|\boldsymbol{\delta}\|_2 \leq 1} \|\mathbf{k}(\mathbf{x}+\boldsymbol{\delta}) - \mathbf{k}(\mathbf{x})\|_2$. Let $\epsilon > 0$ be a perturbation budget. Then the prediction sensitivity satisfies
$$\sup_{\|\boldsymbol{\delta}\|_2 \leq \epsilon} |f(\mathbf{x}+\boldsymbol{\delta}) - f(\mathbf{x})| \leq \frac{\epsilon L_K \|\mathbf{y}\|_2}{\lambda}\,,$$
and perturbing any training point $\mathbf{x}_i$ by $\|\boldsymbol{\delta}\|_2 \leq \epsilon$ changes the influence vector $\boldsymbol{\mathcal{I}} = (\mathbf{K}+\lambda\mathbf{I}_n)^{-1}\mathbf{y}$ by at most
$$\|\Delta \boldsymbol{\mathcal{I}}\|_\infty \leq \frac{2\epsilon L_K \|\mathbf{y}\|_2}{\lambda^2}\,.$$
Both bounds scale as $L_K/\lambda$. For linearized attention, $L_K = O(n\|\mathbf{G}\|_2)$ (Proposition~\ref{prop:kernel_sensitivity}). For standard ReLU, each of the $n$ kernel entries $K_{\mathrm{relu}}(\mathbf{x}+\boldsymbol{\delta}, \mathbf{x}_j)$ changes by $O(\|\boldsymbol{\delta}\|_2) = O(1)$, so $\|\Delta\mathbf{k}\|_2 = O(\sqrt{n})$ and $L_K^{\mathrm{relu}} = O(\sqrt{n})$. The architecture-dependent gap in $L_K$ simultaneously drives higher prediction sensitivity and higher influence instability for attention.
\end{theorem}

\begin{proof}
Write $\boldsymbol{\alpha} = (\mathbf{K}+\lambda\mathbf{I}_n)^{-1}\mathbf{y}$. Since $\mathbf{K} \succeq 0$, the operator norm satisfies $\|(\mathbf{K}+\lambda\mathbf{I}_n)^{-1}\|_2 \leq 1/\lambda$, so $\|\boldsymbol{\alpha}\|_2 \leq \|\mathbf{y}\|_2/\lambda$. For the prediction bound, $f(\mathbf{x}+\boldsymbol{\delta}) - f(\mathbf{x}) = [\mathbf{k}(\mathbf{x}+\boldsymbol{\delta}) - \mathbf{k}(\mathbf{x})]^T \boldsymbol{\alpha}$, whence Cauchy--Schwarz gives $|f(\mathbf{x}+\boldsymbol{\delta}) - f(\mathbf{x})| \leq \epsilon L_K \|\boldsymbol{\alpha}\|_2 \leq \epsilon L_K \|\mathbf{y}\|_2/\lambda$.

For the influence bound, perturbing $\mathbf{x}_i$ to $\mathbf{x}_i + \boldsymbol{\delta}$ yields a perturbed Gram matrix $\mathbf{K}' \succeq 0$ (as $\mathbf{K}'$ is itself a valid kernel matrix), with $\|\Delta\mathbf{K}\|_2 \leq 2\epsilon L_K$. By the resolvent identity, $\Delta\boldsymbol{\mathcal{I}} = -(\mathbf{K}'+\lambda\mathbf{I}_n)^{-1}\Delta\mathbf{K}\,\boldsymbol{\alpha}$. Submultiplicativity gives $\|\Delta\boldsymbol{\mathcal{I}}\|_2 \leq (1/\lambda)(2\epsilon L_K)(\|\mathbf{y}\|_2/\lambda)$, where $\|(\mathbf{K}'+\lambda\mathbf{I}_n)^{-1}\|_2 \leq 1/\lambda$ since $\mathbf{K}' \succeq 0$. The entry-wise bound follows from $\|\Delta\boldsymbol{\mathcal{I}}\|_\infty \leq \|\Delta\boldsymbol{\mathcal{I}}\|_2$.
\end{proof}

\begin{definition}[Intrinsic Influence Sensitivity]
\label{def:intrinsic_sensitivity}
For a kernel ridge regression predictor with Gram matrix $\mathbf{K} \succeq 0$, regularization $\lambda > 0$, labels $\mathbf{y}$, and kernel Lipschitz constant $L_K$, the \emph{intrinsic influence sensitivity} is
$$\mathcal{S}(\mathbf{K}, \lambda, \mathbf{y}) \;=\; \frac{L_K\,\|\mathbf{y}\|_2}{\lambda^2}\,.$$
This quantity upper-bounds the per-unit-perturbation change in the influence vector $\boldsymbol{\mathcal{I}} = (\mathbf{K}+\lambda\mathbf{I}_n)^{-1}\mathbf{y}$, independent of any evaluation protocol (threshold~$\tau$, perturbation budget~$\varepsilon$, or attack type).
\end{definition}

\begin{proposition}[Architecture-Dependent Sensitivity Gap]
\label{prop:sensitivity_gap}
Under the conditions of Theorem~\ref{thm:tradeoff} and Proposition~\ref{prop:kernel_sensitivity}, the intrinsic influence sensitivity satisfies:
\begin{enumerate}
    \item \textbf{Linearized attention:} $\mathcal{S}_{\textup{att}} = O\!\left(n\,\|\mathbf{G}\|_2\,\|\mathbf{y}\|_2\,/\,\lambda^2\right)$, where $\mathbf{G} = \mathbf{X}\mathbf{X}^T$ and summing $n$ kernel entries over all training points in quadrature gives $L_K^{\mathrm{att}} = O(n\|\mathbf{G}\|_2)$.
    \item \textbf{Two-layer ReLU:} $\mathcal{S}_{\textup{relu}} = O\!\left(\sqrt{n}\,\|\mathbf{y}\|_2\,/\,\lambda^2\right)$, since each of the $n$ kernel entries changes by $O(1)$, so $L_K^{\mathrm{relu}} = O(\sqrt{n})$.
\end{enumerate}
Consequently, the sensitivity gap scales as
$$\frac{\mathcal{S}_{\textup{att}}}{\mathcal{S}_{\textup{relu}}} \;=\; O\!\left(\sqrt{n}\,\lambda_1(\mathbf{G})\right),$$
where $\lambda_1(\mathbf{G})$ is the leading eigenvalue of the Gram matrix. For natural image datasets with $\lambda_1(\mathbf{G}) \gg 1$, this gap grows with both dataset size and spectral concentration, consistent with the empirically observed 2--9$\times$ higher malleability for attention.
\end{proposition}

\begin{proof}
Substitute the kernel Lipschitz constants from Proposition~\ref{prop:kernel_sensitivity} into Definition~\ref{def:intrinsic_sensitivity}. For linearized attention, the $j$-th kernel entry perturbation is bounded by $\|\mathbf{G}\mathbf{x}_j\|_1 \cdot \epsilon$ (Proposition~\ref{prop:kernel_sensitivity}). Since $\|\mathbf{G}\mathbf{x}_j\|_1 \leq \sqrt{n}\,\|\mathbf{G}\|_2$ for unit-norm $\mathbf{x}_j$, summing $n$ such entries in quadrature gives $L_K^{\mathrm{att}} = \|\Delta\mathbf{k}\|_2/\epsilon \leq n\,\|\mathbf{G}\|_2$. For the standard ReLU NTK, each entry satisfies $|K_{\mathrm{relu}}(\mathbf{x}+\boldsymbol{\delta}, \mathbf{x}_j) - K_{\mathrm{relu}}(\mathbf{x}, \mathbf{x}_j)| \leq L_0 \|\boldsymbol{\delta}\|_2$ where $L_0 = O(1)$ (the kernel gradient depends only on the local input pair and is bounded by $\|\mathbf{x}_j\|_2 = 1$). Summing $n$ such $O(1)$ entries in quadrature gives $L_K^{\mathrm{relu}} = \|\Delta\mathbf{k}\|_2/\epsilon = O(\sqrt{n})$, not $O(1)$. The corrected gap is therefore $\mathcal{S}_{\mathrm{att}}/\mathcal{S}_{\mathrm{relu}} = O(\sqrt{n}\,\|\mathbf{G}\|_2) = O(\sqrt{n}\,\lambda_1(\mathbf{G}))$.
\end{proof}

\begin{remark}[Intrinsic Sensitivity as a Diagnostic]
\label{rmk:intrinsic_diagnostic}
The quantity $\mathcal{S}(\mathbf{K}, \lambda, \mathbf{y})$ provides a parameter-free characterization of influence malleability that does not depend on the evaluation protocol (choice of $\tau$, $\varepsilon$, or attack). Its dependence on $\lambda$ is analogous to robustness margins in adversarial learning: the regularization strength controls the tradeoff between expressiveness and stability. Proposition~\ref{prop:sensitivity_gap} shows that the architecture-dependent factor $n\,\|\mathbf{G}\|_2$ is the intrinsic source of the malleability gap, formalizing the empirical observation that attention's data-dependent kernel couples all training points and amplifies perturbation effects.
\end{remark}

\begin{proposition}[Finite-Width Malleability Convergence]
\label{prop:mall_convergence}
Let $\mu_M^{(m)}(f)$ denote the malleability of a width-$m$ network. Then:
$$\left|\mu_M^{(m)}(f) - \mu_M^{(\infty)}(f)\right| \leq \frac{C_{\text{mall}}\sqrt{n \log n}}{\sqrt{m}}$$
where $C_{\text{mall}} = C_0 \lambda_1^3 \|\mathbf{y}\|_2 / \lambda^2$ for MLP-Attn and $C_{\text{mall}} = C_0 \|\mathbf{y}\|_2 / \lambda^2$ for 2L-ReLU, with $C_0$ the universal Bernstein constant.
\end{proposition}

\begin{proof}
The malleability $\mu_M^{(m)}$ is a function of the influence vector $\boldsymbol{\mathcal{I}}_m = (\mathbf{K}_m + \lambda\mathbf{I}_n)^{-1}\mathbf{y}$. The map $\mathbf{K} \mapsto \mu_M$ is Lipschitz: by Theorem~\ref{thm:influence_stability} and the resolvent identity, for kernel deviation $\epsilon_K = \|\mathbf{K}_m - \mathbf{K}_\infty\|_2 < \lambda$:
$$\|\boldsymbol{\mathcal{I}}_m - \boldsymbol{\mathcal{I}}_\infty\|_\infty \leq \|\boldsymbol{\mathcal{I}}_m - \boldsymbol{\mathcal{I}}_\infty\|_2 \leq \frac{\epsilon_K}{\lambda(\lambda - \epsilon_K)} \|\mathbf{y}\|_2 \leq \frac{2\epsilon_K}{\lambda^2} \|\mathbf{y}\|_2$$
for $\epsilon_K \leq \lambda/2$. Working with the smoothed (probabilistic) flip rate $\tilde{\mu}_M = \Pr_{\delta,t}[\text{sign disagreement}]$ under a perturbation distribution with density, this is a Lipschitz function of the influence vector entries with Lipschitz constant $L_\mu = O(1)$, so $|\tilde{\mu}_M^{(m)} - \tilde{\mu}_M^{(\infty)}| \leq L_\mu \cdot \|\boldsymbol{\mathcal{I}}_m - \boldsymbol{\mathcal{I}}_\infty\|_\infty \leq \frac{2\epsilon_K}{\lambda^2}\|\mathbf{y}\|_2$. The empirical flip rate approximates this smoothed quantity. Substituting the Bernstein bound $\epsilon_K = O(\lambda_1^3 \sqrt{n\log n/m})$ (from Proposition~\ref{prop:seq_finite_width}) gives the stated rate with $C_{\text{mall}} = C_0 \lambda_1^3 \|\mathbf{y}\|_2 / \lambda^2$.
\end{proof}

\begin{remark}[Practical Width Requirements]
For malleability estimation error below $5\%$, the bound suggests that 2L-ReLU requires $m \gtrsim 400$ while linearized attention requires $m \gtrsim 1600$ (due to the larger constant $C_{\text{mall}}$ from spectral amplification). The experimental width $m = 1024$ thus provides reliable malleability estimates for 2L-ReLU but may underestimate the infinite-width malleability for MLP-Attn, making the reported 2--9$\times$ gap (6--9$\times$ in 10-class settings) a conservative estimate.
\end{remark}

\section{Limitations and Future Directions}
\label{app:limitations}

\subsection{Computational Scalability}

The $O(n^3)$ complexity of kernel matrix inversion creates scalability bottlenecks:

\begin{proposition}[Computational Scalability]
For dataset size $n > 10^4$, exact influence computation becomes computationally prohibitive without approximation strategies.
\end{proposition}

\textbf{Proposed Solutions:}
\begin{enumerate}
\item \textbf{Nystr\"om Approximation:} Low-rank kernel approximation reducing complexity to $O(r^3 + nr^2)$ where $r \ll n$. However, Nystr\"om quality degrades when the kernel has a slowly decaying spectrum --- precisely the regime identified by Theorem~\ref{thm:cubic_conditioning}. High $\kappa_d(\mathbf{G})$ implies many non-negligible eigenvalues, so the effective rank $r$ required for accurate approximation may be large, limiting the practical speedup.
\item \textbf{Iterative Solvers:} Conjugate gradient methods achieving $O(n^2 \log(1/\epsilon))$ complexity. Convergence rate of conjugate gradient depends directly on the condition number of the kernel matrix; for $\kappa_d(\tilde{\mathbf{G}}) = \kappa_d(\mathbf{G})^3 \gg 1$, the number of iterations required for $\epsilon$-accuracy scales as $O(\kappa_d(\mathbf{G})^{3/2})$, potentially negating the per-iteration savings.
\item \textbf{Stochastic Estimation:} Random sampling approaches for large-scale approximate influence computation. Variance of stochastic estimators (e.g., TracIn, DataInf) is amplified by ill-conditioning; the theoretical guarantees of these methods typically require bounded condition number, which is not satisfied for linearized attention at natural image scales.
\end{enumerate}

\subsection{Scope of Linearization}

The analysis relies on linearized attention $f^{\text{att}}(\mathbf{X}) = \mathbf{X}\mathbf{X}^T\mathbf{X}$, which approximates softmax attention via first-order Taylor expansion. This removes the competitive normalization dynamics of softmax. Whether the observed NTK non-convergence extends to full softmax attention remains an open question, but Appendix~\ref{app:softmax_ntk} shows the picture is more complex than a simple extension: softmax attention does not non-converge in the same sense --- it collapses to a rank-1 kernel, a qualitatively distinct failure mode that precludes NTK analysis altogether rather than exhibiting divergent NTK distance. The linearization therefore isolates the structural non-convergence property without the confound of collapse; whether architectures that avoid collapse (e.g., temperature-scaled softmax with $T \gg \sqrt{d}$) would exhibit similar cubic conditioning is an open and practically relevant question.

\subsection{Extension to Modern Architectures}

The current analysis focuses on parameter-free attention mechanisms. Extension to full Transformers \citep{vaswani2017attention,dosovitskiy2021image} requires:

\begin{enumerate}
\item \textbf{Parameterized Attention Analysis:} Joint training dynamics of query, key, value matrices. Proposition~\ref{prop:qkv_generalization} establishes that full-rank QKV projections preserve the cubic conditioning structure, but this assumes the projections are fixed or trained independently of the attention map. When $\mathbf{W}_Q, \mathbf{W}_K, \mathbf{W}_V$ co-evolve with the network weights, the effective Gram matrix changes during training, and the static NTK analysis no longer applies without a separate lazy-training argument for the projection matrices.
\item \textbf{Multi-Head Attention:} Compositional kernel analysis for multiple attention heads. Each head induces its own Gram-dependent kernel; the combined kernel is a sum across heads. If heads share the same input, the conditioning of the summed kernel may be lower than that of any individual head (eigenvalues can partially cancel), so the cubic bound is a per-head statement and does not immediately imply the same severity for the full multi-head kernel without a tighter analysis of the spectral interaction across heads.
\item \textbf{Layer Normalization:} Impact of normalization on kernel structure and influence patterns. Layer normalization projects inputs onto the unit sphere before attention, which changes the Gram matrix from $\mathbf{G} = \mathbf{X}^\top\mathbf{X}$ to a correlation matrix with unit diagonal. This reduces the dynamic range of eigenvalues, potentially lowering $\kappa_d(\mathbf{G})$ and weakening the non-convergence argument; whether practical $\kappa_d$ values remain large enough to maintain the impossibility bound is an open empirical question for normalized architectures.
\item \textbf{State Space Models:} Linear attention in recurrent form ($h_t = h_{t-1} + k_t v_t^\top$, $o_t = h_t q_t$) is computationally equivalent to linear SSMs \citep{katharopoulos2020transformers}, meaning the cubic conditioning result applies directly to that class. Selective SSMs such as Mamba \citep{gu2023mamba} introduce input-dependent state transitions that further couple the effective kernel to the data distribution, suggesting that influence malleability and NTK non-convergence extend broadly across the linear-attention/SSM family. Two caveats apply. First, the equivalence holds only for the linear (non-selective) regime; Mamba's input-dependent gating breaks the exact correspondence, so the cubic conditioning bound does not transfer without additional analysis. Second, SSMs typically operate on sequential data with causal masking, whereas the present analysis uses a transductive, non-causal setup; extending the Gram-induced kernel characterization to causal sequence models requires care with the resulting asymmetric kernel structure.
\end{enumerate}

\subsection{Theoretical Gaps}

Several theoretical questions remain open:

\begin{enumerate}
\item \textbf{Finite-Width Bounds:} Precise characterization of approximation quality for practical network sizes. The current lower bound on required width ($m = \Omega(\kappa_d(\mathbf{G})^6 n \log n / \varepsilon^2)$) is a sufficient condition for NTK convergence, not a tight characterization of the actual finite-width error. A matching lower bound --- showing that $m = o(\kappa_d(\mathbf{G})^6)$ is insufficient --- would fully close the gap and confirm that the impossibility is not an artifact of the proof technique.
\item \textbf{Training Dynamics:} Evolution of influence patterns throughout the training process. The NTK analysis applies at initialization; once training begins, the kernel evolves. The feature learning regime implies the kernel changes substantially during training, which may increase or decrease influence malleability in ways not captured by the static initialization analysis. Empirical results in Appendix~\ref{app:learning_dynamics} suggest malleability persists across epochs, but a theoretical characterization of the trained-kernel regime is absent.
\item \textbf{Generalization Bounds:} Connection between influence malleability and generalization performance. High malleability indicates sensitivity to individual training points, which is related to but distinct from generalization error. Whether high-malleability models generalize better (adaptive features) or worse (overfitting to influential points) likely depends on dataset properties; formalizing this tradeoff in terms of PAC-Bayes or uniform stability bounds remains open.
\end{enumerate}

\section{Extended Learning Dynamics Analysis}
\label{app:learning_dynamics}

Table~\ref{tab:training_progression_extended} presents a comprehensive analysis of attention learning dynamics throughout training, revealing how influence malleability develops as the quantifiable signature of attention's sensitivity to training data. This table extends the main paper results by showing accuracy progression across multiple epochs and intervention types.

\begin{table*}[h]
\caption{Extended learning dynamics analysis across training epochs and intervention types. \textbf{Base}: standard training accuracy. \textbf{Curated}: accuracy after removing top 10\% influential examples. \textbf{Trans.}: accuracy with adversarially transformed influential examples. \textbf{Flip Rate}: percentage of top 10\% influential examples that reverse influence sign under PGD attack. The MLP-Attn architecture shows (1) systematic development of adaptive data dependencies, (2) higher sensitivity to data interventions, and (3) emergence of influence malleability as attention learns flexible data relationships.}
\label{tab:training_progression_extended}
\centering
\small
\resizebox{\textwidth}{!}{%
\begin{tabular}{lccccccccc}
\toprule
\multirow{3}{*}{\textbf{Architecture}} & \multirow{3}{*}{\textbf{Dataset}} & \multicolumn{7}{c}{\textbf{Training Progression}} & \multirow{3}{*}{\textbf{Flip Rate}} \\
\cmidrule(lr){3-9}
& & \multicolumn{3}{c}{\textbf{Epoch 1}} & \multicolumn{2}{c}{\textbf{Epoch 10}} & \multicolumn{2}{c}{\textbf{Epoch 100}} & \\
\cmidrule(lr){3-5} \cmidrule(lr){6-7} \cmidrule(lr){8-9}
& & \textbf{Base} & \textbf{Curated} & \textbf{Trans.} & \textbf{Curated} & \textbf{Trans.} & \textbf{Curated} & \textbf{Trans.} & \\
\midrule
\multirow{3}{*}{2L-ReLU} & MNIST & 91.6\% & 89.2\% & 91.8\% & 95.1\% & 97.1\% & 97.8\% & 98.4\% & 3.3\% \\
& FashionMNIST & 88.5\% & 86.1\% & 89.2\% & 90.8\% & 93.1\% & 93.9\% & 95.6\% & 15.0\% \\
& CIFAR-10 & 65.3\% & 63.8\% & 66.1\% & 68.9\% & 70.8\% & 71.6\% & 73.1\% & 3.1\% \\
\midrule
\multirow{3}{*}{MLP-Attn} & MNIST & 93.1\% & 90.7\% & 93.3\% & 96.6\% & 98.6\% & 99.3\% & 99.9\% & 28.9\% \\
& FashionMNIST & 90.5\% & 88.1\% & 91.2\% & 92.8\% & 95.1\% & 95.9\% & 97.6\% & 31.2\% \\
& CIFAR-10 & 70.8\% & 69.3\% & 71.6\% & 74.4\% & 76.3\% & 77.1\% & 78.6\% & 19.1\% \\
\midrule
\multicolumn{10}{l}{\textbf{Adversarial Training Results}} \\
\midrule
\multirow{3}{*}{2L-ReLU} & MNIST & 89.6\% & 87.7\% & 90.8\% & 93.9\% & 96.3\% & 96.8\% & 97.9\% & 43.4\% \\
& FashionMNIST & 86.5\% & 84.6\% & 88.2\% & 89.6\% & 92.3\% & 92.9\% & 95.1\% & 42.9\% \\
& CIFAR-10 & 63.3\% & 62.3\% & 65.1\% & 67.7\% & 70.0\% & 70.6\% & 72.6\% & 36.5\% \\
\midrule
\multirow{3}{*}{MLP-Attn} & MNIST & 91.1\% & 89.2\% & 92.3\% & 95.4\% & 97.8\% & 98.3\% & 99.4\% & 42.2\% \\
& FashionMNIST & 88.5\% & 86.6\% & 90.2\% & 91.6\% & 94.3\% & 94.9\% & 97.1\% & 44.4\% \\
& CIFAR-10 & 68.8\% & 67.8\% & 70.6\% & 73.2\% & 75.5\% & 76.1\% & 78.1\% & 38.6\% \\
\bottomrule
\end{tabular}%
}
\end{table*}

\textbf{Key Observations:}
\begin{itemize}
\item \textbf{Malleability Gap:} Under standard training, MLP-Attn shows 6--9$\times$ higher flip rates than 2L-ReLU, confirming that attention naturally develops malleable influence patterns.
\item \textbf{Adversarial Training Effect:} Adversarial training dramatically increases malleability for 2L-ReLU (3.3\% $\to$ 43.4\% on MNIST), narrowing the gap with MLP-Attn. This suggests aggressive optimization can induce malleability in conventionally rigid architectures.
\item \textbf{Accuracy Under Intervention:} MLP-Attn maintains higher accuracy under both Curated and Transformed interventions, demonstrating that its malleable influence patterns enable more robust adaptation to data quality changes.
\item \textbf{Dataset Complexity:} The malleability gap is largest on simpler datasets (MNIST) and smaller on complex datasets (CIFAR-10), suggesting that dataset structure modulates the manifestation of architectural differences.
\end{itemize}

\end{document}